\documentclass{article}
\PassOptionsToPackage{numbers}{natbib}
\usepackage[main,final]{neurips_2026}

\usepackage[utf8]{inputenc}
\usepackage[T1]{fontenc}
\usepackage{hyperref}
\usepackage{url}
\usepackage{booktabs}
\usepackage{amsfonts}
\usepackage{amsmath}
\usepackage{amssymb}
\usepackage{amsthm}
\usepackage{nicefrac}
\usepackage{microtype}
\usepackage{xcolor}
\usepackage{graphicx}
\usepackage{algorithm}
\usepackage{algorithmic}
\usepackage{multirow}
\usepackage{soul}
\usepackage{wrapfig}
\usepackage{capt-of}
\usepackage{hyperref}

\newtheorem{proposition}{Proposition}
\newtheorem{remark}{Remark}

\newcommand{\motif}{FAFM}
\newcommand{\lvel}{\mathcal{L}_{\mathrm{vel}}}
\newcommand{\lfm}{\mathcal{L}_{\mathrm{FM}}}

\usepackage{etoolbox}

\newcommand{\tablefontsize}{\small} % 可改成 \footnotesize / \scriptsize

\AtBeginEnvironment{tabular}{\tablefontsize}
\AtBeginEnvironment{tabular*}{\tablefontsize}

\newtheorem{theorem}{Theorem}

\newtheorem{lemma}{Lemma}

\title{Frequency‑Aware Flow Matching for Continuous and Consistent Robotic Action Generation}

\author{
Jianing Guo\textsuperscript{1, 4},
Fangzheng Chen\textsuperscript{1, 4}, 
Zihao Mao\textsuperscript{1},
Wong Lik Hang Kenny\textsuperscript{3},
Zhenhong Wu\textsuperscript{1}, 
Yu Li\textsuperscript{2, 4}, \\ \bfseries
Yishuai Cai\textsuperscript{4}, 
Yuanpei Chen\textsuperscript{2, 4}, 
Yikun Ban\textsuperscript{1}, 
Kai Chen\textsuperscript{3}, 
Qi Dou\textsuperscript{3},
Xianglong Liu\textsuperscript{1, 5, 6}, \\ \bfseries
Huijie Zhao\textsuperscript{1},
Simin Li\textsuperscript{1, 3 \thanks{Corresponding Author. E-mails: lisiminsimon@buaa.edu.cn.}}, 
Yaodong Yang\textsuperscript{2, 4}
}

\begin{document}

\maketitle
\vspace{-0.3in}
\begin{flushleft}
\textsuperscript{1}{Beihang University}, \textsuperscript{2}{Peking University}, \textsuperscript{3}{The Chinese University of Hong Kong}, \textsuperscript{4}{PKU-Psibot Lab}, \textsuperscript{5}{Zhongguancun Laboratory}, \textsuperscript{6}{Hefei Comprehensive National Science Center}
\end{flushleft}

%% ============================================================
\begin{abstract}

% Flow matching has been adopted as the standard paradigm for robotic manipulation due to its strong expressive power. However, existing approaches operate on discretized action chunks, which struggle with training data with heterogeneous frequencies and often produce temporally inconsistent actions that leads to control instability. In  this paper, we learn flow-matching policy that outputs continuous actions directly. To handle input with heterogeneous frequency, we model continuous actions by transforming discrete action input to frequency domain using the discrete cosine transform (DCT), performing flow matching in frequency coefficients and synthesize continuous actions via the inverse DCT transform. To generate smooth and gentle actions, we regularize the first‑order temporal derivative of the action sequence, which corresponds to a Sobolev‑type regularization that suppresses high‑frequency errors and discourages abrupt action changes. Our method do not introduce additional parameter, incurs negligible computational overhead, and speeds up convergence. Experimental results on synthetic example, obstacle avoidance, and soft‑body manipulation tasks shows our method exhibits better multimodal expressive power, higher success rates, greater solution diversity, faster convergence and reduced action jitter comparing with strong baselines.

Flow matching has emerged as a standard paradigm for robotic manipulation owing to its strong expressive power for modelling complex, multimodal action distributions, alongside similar approaches like diffusion policy. However, existing methods rely on discretized action chunks, making them brittle to demonstrations collected at heterogeneous control frequencies and prone to temporally inconsistent actions that degrade control stability. In this paper, we propose Frequency-Aware Flow Matching (FAFM), which outputs continuous, temporally consistent actions. To handle heterogeneous frequency input, we transform discrete action sequences into the frequency domain with the discrete cosine transform (DCT), perform flow matching over the resulting coefficients, and reconstruct continuous actions via cosine basis expansion. To generate temporally consistent actions, we regularize the first-order temporal derivative to promote smooth actions. This corresponds to a Sobolev-type constraint that suppresses high-frequency errors and discourages abrupt action changes. Our FAFM is simple, introduces no additional network parameters and applies to standalone flow-matching policies and vision-language action models. Across synthetic toy benchmark, obstacle avoidance, LapGym, and LIBERO, FAFM improves success rates, multimodal expressivity, motion smoothness, convergence speed, robustness to mechanical bias and mixed-frequency input. These gains are consistent when deployed on a real-world Franka robot. Code available at 
\url{https://github.com/gakakulicc/FAFM}.

\end{abstract}

%% ============================================================
\section{Introduction}
\label{sec:intro}

% Flow matching policies have enabled robots strong capability to capture complex, multi-modal action distributions. Due to its strong representation capability, it serves as the action generation head of multiple robot foundation models, which granted flexible, general, and dexterous manipulation behavior in real-world tasks. To enable real-time inference, flow matching policies typically output an action chunk, which is a sequence of consecutive actions at discrete time steps being executed at a fixed frequency.

Flow matching~\citep{lipmanflow, liu2022rectified} endows robot policies to represent complex, multimodal action distributions. Owing to this expressive power, flow matching has been adopted as the action head in several robot foundation models~\citep{black2024pi_0, intelligence2025pi_, shukor2025smolvla, zheng2025x, zhang2026a1, cheang2025gr}, enabling flexible, generalizable, and dexterous manipulation in real‑world settings. To support real‑time inference, such policies typically predict an action chunk~\citep{zhao2023learning}, which is a sequence of consecutive actions at discrete time steps, executed at a fixed control frequency.

However, real-world robot motions are inherently smooth and continuous. As illustrated in Fig. 1, discretizing trajectories into action chunks introduces problems at both training and inference time. During training, the policy must reconcile demonstration data recorded at heterogeneous control frequencies that potentially cause gradient conflict. For example, the Open X-Embodiment dataset~\citep{o2024open} spans a wide range of control frequencies, from 3 Hz (RT-1~\citep{brohan2022rt}) to 50 Hz (ALOHA~\citep{zhao2023learning}), yet the policy is required to produce action chunks at a single fixed frequency. This discrepancy means that frequency information in the training data is largely overlooked. Consequently, the same robot executing the same physical trajectory at different frequencies will produce action chunks of different lengths and magnitudes, each of which is nonetheless a valid supervision signal.

% Presenting such conflicting data simultaneously during training leads to gradient conflict. 
% \textcolor{red}{Although the recent model $\pi_{0.7}$ conditions on episode speed as metadata during training}, gradient conflict arising from frequency heterogeneity is not fully resolved.

% However, real‑world robot motion is inherently smooth and temporally consistent. As illustrated in Fig. 1, discretizing trajectories into action chunks introduces challenges during both training and inference. During training, the policy must reconcile demonstration data recorded at heterogeneous control frequencies, which can induce gradient conflict. For example, the Open X‑Embodiment dataset spans a wide range of control frequencies, from 3 Hz (RT‑1) to 50 Hz (ALOHA), yet the policy is trained to output action chunks at a single, fixed frequency. As a result, frequency information present in the demonstrations is entirely discarded. Consequently, the same robot executing the same physical trajectory at different frequencies will produce action chunks of different lengths and magnitudes, yet each of them is nonetheless a valid supervision signal. Training on such self-contradictory data leads to gradient conflict. Although the recent model $\pi_{0.7}$ condition on episode speed as auxiliary metadata, this approach does not fully resolve gradient conflict arising from frequency heterogeneity.

During inference, actions at discrete timesteps are evaluated independently, without accounting for inter-timestep consistencies. Consequently, errors at adjacent timesteps may individually appear small, yet correspond to mutually conflicting motion directions. Executing such actions produces jittery, discontinuous motion that compromises control stability~\citep{scheikl2024movement, nguyen2025flowmp, yang2026abpolicy}. Although this may not directly reduce task success rates in rigid-body settings, it poses serious risks in soft-body manipulation. For instance, during liquid pouring, abrupt motion discontinuities can cause unintended spillage; in robotic surgery, jittery instrument motion can drive soft tissue into path-dependent deformation states that are difficult or impossible to reverse~\citep{attanasio2021autonomy}.

\begin{figure}[t]
    \centering
    \includegraphics[width=0.9\textwidth]{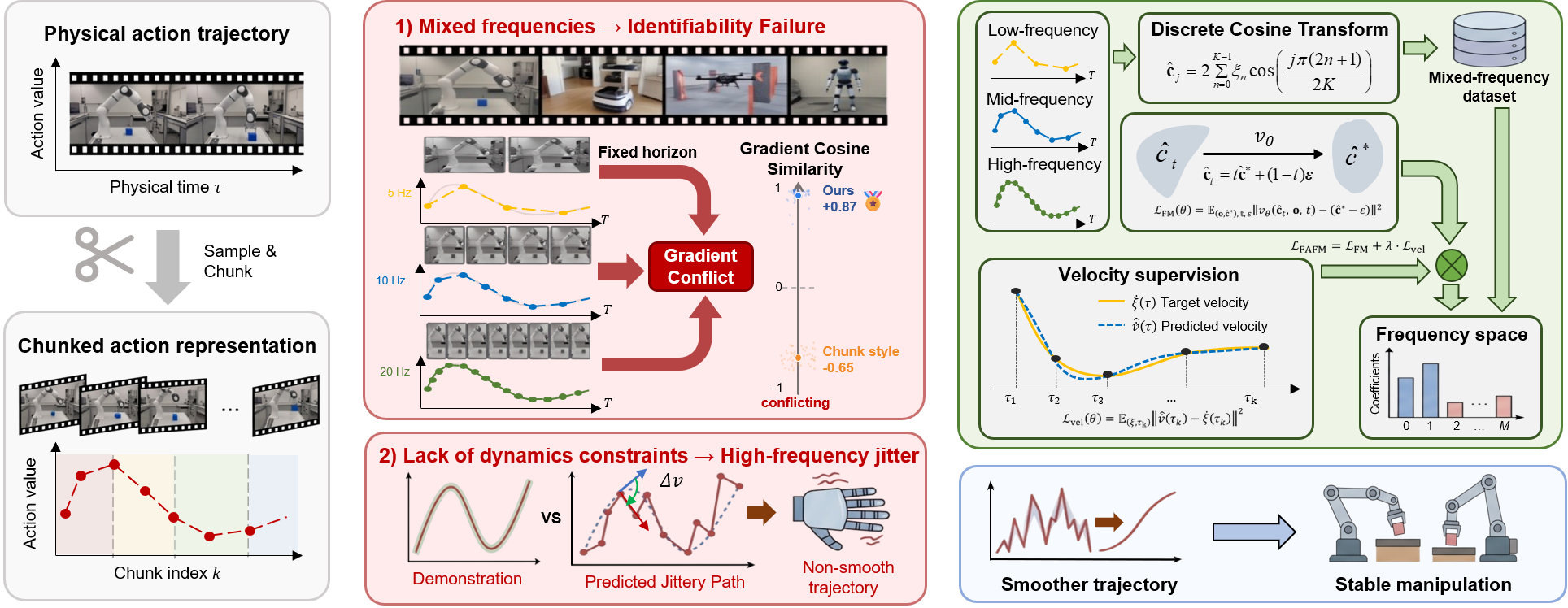}
    % \vspace{-0.4pt}
    \caption{Introduction overview. Action chunk discretization causes Identifiability Failure and high-frequency jitter, which FAFM resolves via frequency-domain flow matching and velocity supervision.}
    % \caption{Introduction overview. Left: real-world robot actions are inherently continuous trajectories, while existing policies typically represent them as discrete action chunks. Middle: this mismatch produces two key issues, namely Identifiability Failure under mixed-frequency training and high-frequency jitter caused by insufficient dynamics constraints. Right: our method addresses these issues by learning in the frequency domain and imposing velocity supervision, producing smoother trajectories and more stable manipulation.}
    \label{fig:intro}
    \vspace{-0.2in}
\end{figure}

% \noindent\begin{minipage}{\textwidth}
%     \centering
%     \includegraphics[width=\textwidth]{figure/intro_v1.png}
%     \vspace{-4pt}
%     \captionof{figure}{Introduction overview. Left: real-world robot actions are inherently continuous trajectories, while existing policies typically represent them as discrete action chunks. Middle: this mismatch produces two key issues, namely Identifiability Failure under mixed-frequency training and high-frequency jitter caused by insufficient dynamics constraints. Right: our method addresses these issues by learning in the frequency domain and imposing velocity supervision, producing smoother trajectories and more stable manipulation.}
%     \label{fig:intro}
% \end{minipage}
% \vspace{-6pt}

In this paper, we propose learning continuous and temporally consistent actions directly via Frequency-Aware Flow-Matching (FAFM). To accommodate discrete action trajectories collected at arbitrary sampling frequencies in training time, we transform the data into the frequency domain using the discrete cosine transform (DCT) \cite{ahmed1974discrete}. Although DCT was recently introduced to the robotic community for real-time execution in FAST \cite{pertsch2025fast}, we exploit it here as a basis function that naturally separates components at different frequencies. Flow matching is then performed in the frequency domain, and the predicted DCT coefficients parameterize a continuous-time action trajectory via a cosine basis expansion, which can be evaluated at arbitrary temporal resolutions. To further promote temporal consistency during inference, we additionally apply flow matching to the first-order temporal derivative of the action sequence. Because the DCT provides a continuous action representation in frequency space, this derivative is well defined and accurate. This formulation corresponds to a Sobolev-type regularization, in which errors at higher frequencies are penalized quadratically stronger than lower frequencies. Consequently, the resulting loss suppresses high-frequency errors and effectively discourages abrupt changes in the generated actions.

Overall, our FAFM is simple yet effective. It introduce no additional networks or learnable parameters, and applies to standalone flow matching policy and VLAs. Experiments on synthetic toy benchmarks, obstacle avoidance, LapGym~\citep{scheikl2023lapgym}, and LIBERO~\citep{liu2023libero} demonstrate that our FAFM consistently achieves higher success rate than strong baselines. On synthetic toy benchmark consisting of a pair of antisymmetric sinusoidal trajectories, our approach uniquely captures the correct multimodal structure. In obstacle avoidance, our FAFM simultaneously achieves high solution diversity and strong overall performance. On LapGym, a robotic surgery benchmark that requires soft-body manipulation, our method achieves smooth control and faster convergence. On LIBERO benchmark for VLAs, our FAFM demonstrate superior motion smoothness, robustness to mechanical bias and mixed-frequency input. The superiority is consistent when deployed to real-world Franka robot.

\textbf{Contribution.} Our contributions are twofold. (1) We introduce FAFM, a frequency‑aware flow matching method for continuous and temporally consistent action generation. 
% FAFM is parameter-free and applies to standalone flow matching and VLAs.
FAFM applies to standalone flow matching and VLAs.
(2) On various benchmarks, our FAFM achieves higher success rates, improved multimodal expressivity, faster convergence, higher motion smoothness and robustness against mechanical bias and heterogeneous input frequencies.

\vspace{-0.1in}
\section{Related Work}
\label{sec:related}

\textbf{Generative Policy Learning for Robot Action Generation.} Recent advances in diffusion and flow‑matching methods have reshaped generative action policies. For vision-based control, diffusion policy~\citep{chi2025diffusion} first introduced the paradigm of generating fixed‑horizon action chunks using denoising diffusion probabilistic models (DDPMs~\citep{ho2020denoising, song2020score}), which has since become a standard approach for continuous robot control. 3D Diffusion Policy~\citep{ze20243d} extended this framework to point‑cloud observations, demonstrating that explicit three‑dimensional representations substantially improve generalization in dexterous manipulation tasks, with follow-ups including ~\citep{urain2023se, reuss2023goal, wang2024equivariant, scheikl2024movement}. More recently, flow‑matching–based policies have emerged as a more inference‑efficient alternative, achieving action quality comparable to diffusion methods while significantly reducing sampling latency~\citep{zhang2024affordance}. FlowPolicy~\citep{zhang2025flowpolicy} further extended flow matching to 3D manipulation settings, confirming its effectiveness under complex spatial interactions, with similar works including~\citep{xu2024flow, braun2024riemannian, yang2026abpolicy, park2025acg}.

Parallel developments have occurred in large robot foundation models. Early systems such as RT‑1~\citep{brohan2022rt}, RT‑2~\citep{zitkovich2023rt} and OpenVLA~\citep{kim2024openvla} predict discrete action tokens autoregressively and subsequently decode them into continuous control signals, trading off expressivity for scalability. More recent vision‑language‑action (VLA) models including Octo~\citep{team2024octo}, CogACT~\citep{li2024cogact}, RDT-1B~\citep{liu2024rdt}, DexGraspVLA~\citep{zhong2026dexgraspvla} and GR00T series~\citep{bjorck2025gr00t} etc have shifted toward diffusion‑based action heads to improve multimodal expressivity, robustness, and cross‑embodiment generalization. Flow‑matching–based action generation was later adopted in $\pi_0$~\citep{black2024pi_0}, demonstrating both improved inference efficiency and strong control performance; this design has since been inherited by successors such as $\pi_{0.5}$~\citep{intelligence2025pi_}, $\pi^{*}_{0.6}$~\citep{intelligence2025pi}, $\pi_{0.7}$~\citep{intelligence2026pi} , SmolVLA~\citep{shukor2025smolvla}, X-VLA~\citep{zheng2025x} and A1~\citep{zhang2026a1}, which now represents the state of the art for high‑performance continuous action generation in robot foundation models.

\textbf{Improving Action Representation for Robot Control.} For autoregressive actions, FAST \citep{pertsch2025fast} applies DCT to compress high-frequency action chunks into discrete tokens, enabling efficient action execution. NIAF~\citep{liu2026neural} extends the parallel decoding regression paradigm introduced by OpenVLA-OFT~\citep{kim2025fine} with implicit neural representations for continuous action function regression. However, autoregressive methods have limited capability in representing multimodal actions, thus later approaches moves towards diffusion-based methods~\citep{chi2025diffusion}. 
For diffusion policies, MPD ~\citep{scheikl2024movement}  (and its consistency-distilled acceleration FRMD~\citep{shi2025frmd}) represent actions as low-dimensional probabilistic weights over temporal basis functions (ProDMP~\citep{li2023prodmp}) for surgical tasks, generating smooth trajectories via globally smooth basis functions. However, this global parameterization couples all local trajectory decisions into a single weight vector, smoothing away local multimodal distinctions. FreqPolicy ~\citep{zhong2025freqpolicy} performs diffusion in a learned latent space structured by DCT frequency bands, but the learned latent entangles frequency content with observation features, causing mode leakage in multimodal settings. As for flow matching policies, SFP~\citep{jiang2025streaming} aligns the flow matching integration time with robot execution time, streaming actions on-the-fly at each ODE step. While achieving faster policy execution and better multi-modal representation capability, SFP operates under a receding horizon scheme and does not capture the joint distribution over future actions globally. ABPolicy~\citep{yang2026abpolicy} adopts flow matching in a B-spline control-point space to enforce trajectory smoothness within and across action chunks, with asynchronous inference for real-time control.

\vspace{-0.1in}
\section{Method}
\label{sec:method}

% Our approach makes two targeted changes to the chunk-based flow-matching pipeline. \emph{First}, we parameterize actions as DCT coefficients rather than discrete position sequences: the resulting targets are anchored to physical time and thus independent of control frequency, eliminating the gradient conflict formalized in Proposition~\ref{prop:identifiability}. \emph{Second}, we supervise the analytic velocity of the DCT-decoded trajectory. Together, these changes amount to optimizing a weighted $H^1$ Sobolev norm that penalizes high-frequency coefficient errors quadratically, giving rise to smooth actions, accelerated convergence, and robustness to execution noise.
% Our approach makes two targeted changes to the chunk-based flow-matching pipeline. \emph{First}, we parameterize actions as DCT coefficients anchored to physical time, eliminating the gradient conflict under heterogeneous control frequencies (Proposition~\ref{prop:identifiability}). \emph{Second}, we supervise the analytic velocity of the DCT-decoded trajectory, together yielding smooth actions, accelerated convergence, and robustness to execution noise.

In this section, we propose frequency-aware flow matching to generate continuous and temporally consistent actions. Our overall approach are illustrated in Figure~\ref{fig:method}.

\begin{figure*}[t]
\centering
\includegraphics[width=0.85\textwidth]{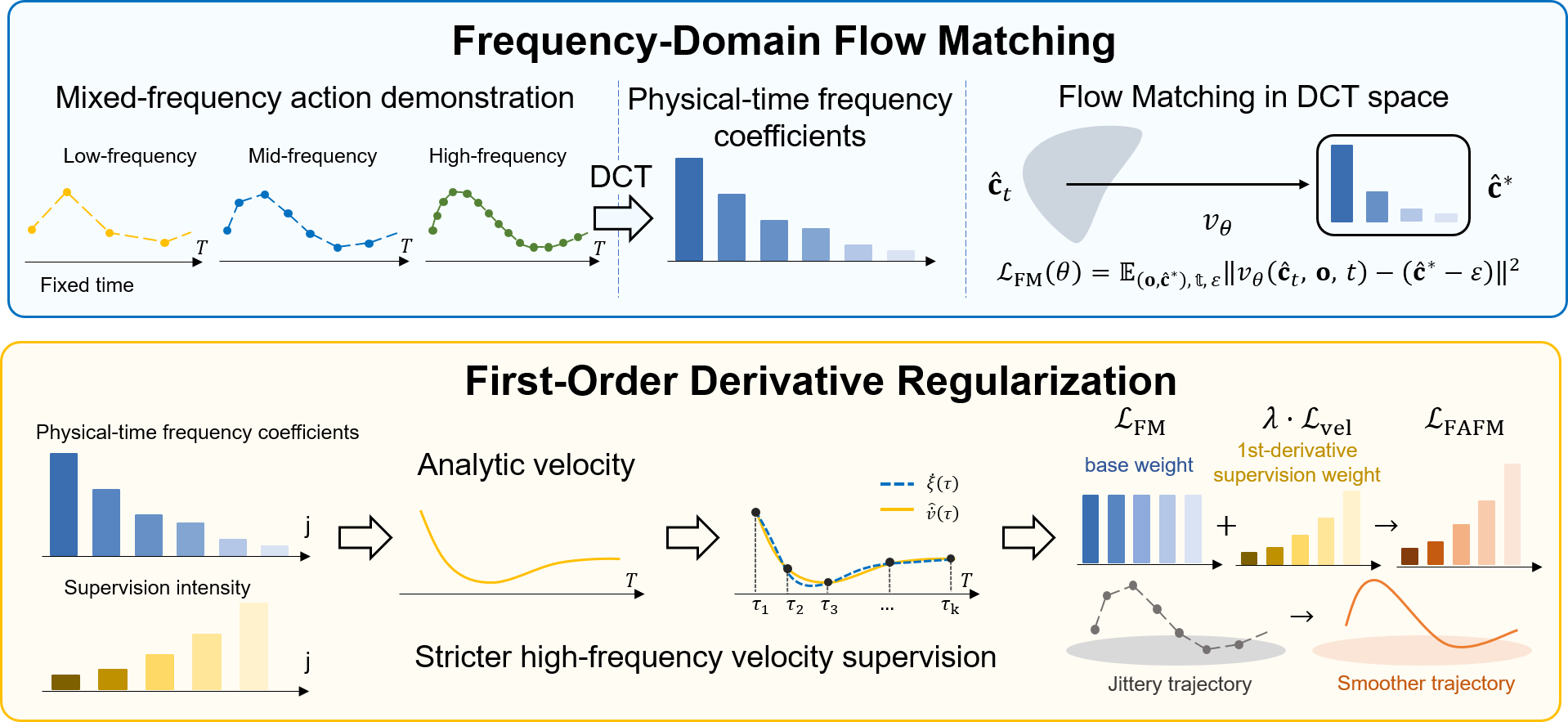}
\caption{Overview of \motif{}. Demonstration trajectories are mapped to DCT coefficients anchored to physical time, flow matching is performed in coefficient space, and analytic temporal-derivative supervision regularizes the continuous action for smooth execution.}
\label{fig:method}
\vspace{-0.2in}
\end{figure*}

% In this section, we propose frequency-aware flow matching to generate continuous and temporally consistent actions. To produce continuous action, we transform heterogeneous frequency input into frequency domain using DCT, and perform flow matching on that basis, which handles training data recorded at different frequencies. To produce temporally consistent actions, we regularize the first-order derivative of the DCT trajectory, leading to smooth and gentle manipulation behavior.

% handle heterogeneous frequency data, we transform them into frequency domain using DCT, and perform flow matching on that basis. To ensure 

% To handle heterogeneous frequency in training data, we perform DCT for .

% Our approach makes two targeted changes to the chunk-based flow-matching pipeline. \emph{First}, we parameterize actions as DCT coefficients anchored to physical time, eliminating mixed-frequency step-index ambiguity. \emph{Second}, we supervise the analytic velocity of the DCT-decoded trajectory, together yielding smooth actions, accelerated convergence, and robustness to execution noise.

%% ------------------------------------------------------------------
\textbf{Problem Formulation.}
We consider robot action generation as learning $p(\mathbf{A}| \mathbf{o})$ over action chunks $\mathbf{A} = [a_0, \ldots, a_{K-1}] \in \mathbb{R}^{K \times d}$ given observation $\mathbf{o}$. Without loss of generality, we use the conditional flow matching paradigm used in $\pi$ series \cite{black2024pi_0}, which is widely used in modern robot foundation models, and can be extended to similar paradigms like diffusion policy~\citep{chi2025diffusion} easily. In conditional flow matching, given $\varepsilon \sim \mathcal{N}(\mathbf{0}, \mathbf{I})$ and target chunk $\mathbf{A}$, one trains a velocity field $v_\theta$ via:
\begin{equation}
  \lfm(\theta)
  \;=\;
  \mathbb{E}_{(\mathbf{o},\mathbf{A}),\,t,\,\varepsilon}
  \bigl\|
    v_\theta(\mathbf{A}^t,\, \mathbf{o},\, t)
    - (\mathbf{A} - \varepsilon)
  \bigr\|^2,
  \qquad
  \mathbf{A}^t = t\mathbf{A} + (1-t)\varepsilon.
  \label{eq:fm}
\end{equation}
Let $p_{\mathcal{D}}$ denote the demonstration distribution over observations, trajectories $\xi:[0,T_\xi]\to\mathbb{R}^d$, and control frequencies $f$; define chunk length $K(\xi,f)=\lfloor T_\xi f\rfloor$ and action $a_k(\xi,f)=\xi(k/f)$ for $k=0,\ldots,K-1$. The chunk is executed at control frequency $f_\xi$, so action $a_k$ corresponds to physical time $\tau_k = k/f_\xi$, with step index $k$ serving as its positional embedding.

\subsection{Continuous Action Generation via Frequency-Domain Flow Matching}
\label{sec:method_dct}

% \begin{wrapfigure}{r}{0.3\linewidth}
%   \vspace{-0.35in}
%   \label{fig:intro_gcs_simple}
%   \centering
%   \includegraphics[width=\linewidth]{figure/box_toy_gradient_conflict.png}
%   \caption{Gradient Conflict}
%   \vspace{-1.0em}
% \end{wrapfigure}

% Training robot foundation models inevitably requires huge amount of data. However, due to the inherent setup of different embodiments, these data are oftentimes not sampled at the same frequency. For example, in Open X-Embodiment dataset, trajectories are recorded at 3 Hz in RT-1 and 50 Hz in ALOHA. However, such difference in control frequencies are not considered in early robot foundation models like $\pi_{0}$ and $\pi_{0.5}$, until very recently $\pi_{0.7}$ includes it as meta data. However, consider a trajectory recorded at different frequencies, learning on such trajectory results in same observation and action as input, yet outputting different actions that corresponds to the input video frequencies, introducing gradient conflict in training. As shown in Fig. xxx, the issue is not fully solved even for $\pi_{0.7}$ that conditions on episode speed. Formally, we find learning on heterogeneous frequency is ill-posed.

Training robot foundation models requires vast amounts of data. Yet because different robot embodiments use different sensing and control stacks, trajectories are often collected at mismatched temporal resolutions. For example, in the Open X-Embodiment dataset~\citep{o2024open}, RT-1~\citep{brohan2022rt} trajectories are recorded at 3 Hz, whereas ALOHA~\citep{zhao2023learning} trajectories are recorded at 50 Hz. Early robot foundation models, such as $\pi_{0}$~\citep{black2024pi_0} and $\pi_{0.5}$~\citep{intelligence2025pi_}, typically ignored these differences.
% ; only recently has $\pi_{0.7}$ incorporated the episode speed as metadata. 
Nevertheless, learning from trajectories logged at heterogeneous frequencies can be problematic: identical observation–action histories may correspond to different next-action targets once the data are indexed in time, which can induce conflicting gradient signals during optimisation. 
% As shown in Fig. xxx, conditioning on episode speed in $\pi_{0.7}$ does not fully resolve this issue. 
More generally, we argue that standard step-indexed action chunk learning from heterogeneous frequency training data can be ambiguous.

% Large-scale robot datasets aggregate demonstrations collected at heterogeneous control frequencies. For example, Open X-Embodiment spans 3\,Hz (RT-1) to 50\,Hz (ALOHA). Standard chunk-based flow matching, however, indexes each action token by its step $k$, implicitly mapping $a_k$ to physical time $\tau_k = k/f_\xi$ through positional embedding. Let $p_{\mathcal{D}}$ denote the demonstration distribution over observations, trajectories $\xi:[0,T_\xi]\to\mathbb{R}^d$, and control frequencies $f$, with $K=\lfloor T_\xi f\rfloor$ and $a_k=\xi(k/f)$. The standard objective trains $v_\theta$ via
% \begin{equation}
%   \lfm(\theta)
%   \;=\;
%   \mathbb{E}_{(\mathbf{o},\mathbf{A}),\,t,\,\varepsilon}
%   \bigl\|
%     v_\theta(\mathbf{A}^t,\, \mathbf{o},\, t)
%     - (\mathbf{A} - \varepsilon)
%   \bigr\|^2,
%   \qquad
%   \mathbf{A}^t = t\mathbf{A} + (1-t)\varepsilon.
%   \label{eq:fm}
% \end{equation}
% Under heterogeneous frequencies, the same step $k$ can map to distinct physical moments $\xi(k/f_{\mathrm{slow}})\neq \xi(k/f_{\mathrm{fast}})$ under the same positional embedding (Fig.~\ref{fig:intro_gcs_simple}). This creates an identifiability failure that persists even with more data or model capacity:

\begin{proposition}[Gradient Conflict in Mixed-Frequency Training]
\label{prop:identifiability}
Consider the step-indexed objective
\begin{equation}
    \mathcal{J}(\hat{a})
    \;=\;
    \mathbb{E}_{(\mathbf{o},\xi,f) \sim p_{\mathcal{D}}}\;
    \mathbb{E}_{k \sim \mathrm{Unif}\{0,\ldots,K(\xi,f)-1\}}
    \Bigl[
    \bigl\|
    \hat{a}(k \mid \mathbf{o}) - \xi(k/f)
    \bigr\|^{2}
    \Bigr].
    \label{eq:step-loss}
\end{equation}
If $f \mid \mathbf{o}$ is non-degenerate, there exist
$f_1 \neq f_2$ in the support of
$p_{\mathcal{D}}(f\mid\mathbf{o})$, the Bayes-optimal solution
\begin{equation}
    \hat{a}^{\star}(k \mid \mathbf{o})
    \;=\;
    \mathbb{E}_{(\xi,f)\sim p_{\mathcal{D}}(\cdot\mid\mathbf{o},\,k<K(\xi,f))}
    \!\bigl[\,\xi(k/f)\,\big|\,\mathbf{o},\,k\,\bigr]
    \label{eq:bayes-optimum}
\end{equation}
averages evaluations of $\xi$ at \emph{distinct} physical times $k/f_1 \neq k/f_2$, producing a physically infeasible target that lies on no single demonstration trajectory in the support of $p_{\mathcal{D}}(\xi\mid\mathbf{o})$.

Proof sketch: The $L^{2}$-optimal predictor is the conditional mean $\hat{a}^{\star}(\mathbf{o},k)=\mathbb{E}[\xi(k/f)\mid\mathbf{o},k]$. 
% , which under non-degenerate $f\mid\mathbf{o}$ averages $\xi$ at distinct physical times $k/f_{1}\neq k/f_{2}$ and generically lies off every demonstration trajectory. 
As $\hat{a}^{\star}$ depends only on $p_{\mathcal{D}}$, no amount of scaling, or optimization can remove this failure. See Appendix~\ref{app:proof_identifiability}.
\end{proposition}

% We work in the Sobolev space $H^1([0,T_\xi];\mathbb{R}^d)$---the
% space of trajectories with square-integrable first
% derivatives---with inner product
% \begin{equation}
%     \langle f,g\rangle_{H^1}
%     \;=\;
%     \langle f,g\rangle_{L^2}
%     \;+\;
%     \langle f',g'\rangle_{L^2},
%     \label{eq:h1_inner}
% \end{equation}
% and expand trajectories in the $L^2$-orthonormal cosine basis
% $\phi_j(\tau)=\alpha_j\cos(\omega_j\tau)$,
% $\omega_j=j\pi/T_\xi$,
% $\alpha_0=\sqrt{1/T_\xi}$, $\alpha_j=\sqrt{2/T_\xi}$ for $j\geq 1$,
% which diagonalizes $-\partial_\tau^2$ under Neumann boundary
% conditions.
% The DCT coefficients $\hat{c}_j$ are precisely the $L^2$ inner
% products of $\xi$ with these basis functions; their structure
% directly shapes the Sobolev training objective in
% Section~\ref{sec:method_loss}.

By contrast, a frequency-domain representation can accommodate trajectories sampled at different rates. We therefore replace step-indexed action chunks with DCT representations. 
To construct such a representation, we model the physical trajectory as an element of $\mathcal{H}:=H^1([0,T_\xi];\mathbb{R}^d) :=\{f\in L^2(0,T_\xi;\mathbb{R}^d)\mid f'\in L^2(0,T_\xi;\mathbb{R}^d)\}$ with inner product
\begin{equation}
    \langle f,g\rangle_{H^1}
    \;=\;
    \langle f,g\rangle_{L^2}
    \;+\;
    \langle f',g'\rangle_{L^2},
    \label{eq:h1_inner}
\end{equation}
and expand in the $L^2$-orthonormal cosine basis $\phi_j(\tau)=\alpha_j\cos(\omega_j\tau)$,
$\omega_j=j\pi/T_\xi$, $\alpha_0=\sqrt{1/T_\xi}$, $\alpha_j=\sqrt{2/T_\xi}$ for $j\geq 1$, which diagonalizes $-\partial_\tau^2$ under Neumann boundary conditions. The target trajectory admits the expansion $\xi^*=\sum_{j\geq 0}c_j^*\phi_j$ with $c_j^*=\langle\xi^*,\phi_j\rangle_{L^2}$.
% we expand physical trajectories in the $L^2$-orthonormal cosine basis of $H^1([0,T_\xi];\mathbb{R}^d)$: $\phi_j(\tau)=\alpha_j\cos(\omega_j\tau)$, $\omega_j=j\pi/T_\xi$, $\alpha_0=\sqrt{1/T_\xi}$, $\alpha_j=\sqrt{2/T_\xi}$ for $j\geq 1$ (see Appendix~\ref{app:proof_h1_zihao} for the full $H^1$ setup). 
% ------------ Origin ------------------
The DCT coefficients are the $L^2$ projections of $\xi$ onto this basis; their structure directly shapes the Sobolev objective in Section~\ref{sec:method_loss}. To remove dependence on step index, we fix the physical chunk duration $T_\xi$ and represent each trajectory by its DCT-II~\citep{strang1999discrete} coefficient vector $\hat{\mathbf{c}}\in\mathbb{R}^{(M+1)\times d}$, $M\ll K$:
\begin{equation}
  \hat{c}_j = \frac{2}{K}\sum_{n=0}^{K-1}\xi_n
    \cos\Bigl(\frac{j\pi(2n+1)}{2K}\Bigr),
  \quad j=0,\ldots,M,
  \label{eq:dct2_forward}
\end{equation}
where $\xi_n=\xi^*(n/f_\xi)$ and $K=\lfloor T_\xi f_\xi\rfloor$
varies with the demonstration frequency.
Because $\frac{j\pi(2n+1)}{2K}=\frac{j\pi\tau_n}{T_\xi}$ depends
only on physical time $\tau_n=n/f_\xi$,
Eq.~\eqref{eq:dct2_forward} is a quadrature rule for the
frequency-independent integral
\begin{equation}
  c_j^* \;=\; \frac{2}{T_\xi}\int_0^{T_\xi}
    \xi^*(\tau)\cos\Bigl(\frac{j\pi\tau}{T_\xi}\Bigr)d\tau,
  \qquad
  \hat{c}_j = c_j^* + O(1/K).
  \label{eq:dct_limit}
\end{equation}
% -------------- Origin ---------------------
% {\color{blue}The DCT coefficients are the $L^2$ projections of $\xi$ onto this basis; their structure directly shapes the Sobolev objective in Section~\ref{sec:method_loss}. To remove dependence on step index, we fix the physical chunk duration $T_\xi$ and represent each trajectory by its DCT-II~\citep{strang1999discrete} coefficients. First, we compute the standard unnormalized DCT-II sum $\tilde{\mathbf{c}}\in\mathbb{R}^{(M+1)\times d}$, $M\ll K$:
% \begin{equation}
%   \tilde{c}_j = 2\sum_{n=0}^{K-1}\xi_n
%     \cos\Bigl(\frac{j\pi(2n+1)}{2K}\Bigr),
%   \quad j=0,\ldots,M,
%   \label{eq:dct2_forward}
% \end{equation}
% where $\xi_n=\xi^*(n/f_\xi)$ and $K=\lfloor T_\xi f_\xi\rfloor$ varies with the demonstration frequency. We then define the normalized physical-time coefficient as $\hat{c}_j := \tilde{c}_j / K$. 
% Because $\frac{j\pi(2n+1)}{2K}=\frac{j\pi\tau_n}{T_\xi}$ depends only on physical time $\tau_n=n/f_\xi$, the normalized coefficient $\hat{c}_j$ serves as a quadrature rule for the frequency-independent integral:
% \begin{equation}
%   c_j^* \;=\; \frac{2}{T_\xi}\int_0^{T_\xi}
%     \xi^*(\tau)\cos\Bigl(\frac{j\pi\tau}{T_\xi}\Bigr)d\tau,
%   \qquad
%   \hat{c}_j = c_j^* + O(1/K).
%   \label{eq:dct_limit}
% \end{equation}
% }

Since $c_j^*$ depends on $\xi^*$ and $T_\xi$, instead of 
$f_\xi$, two demonstrations of the same physical motion at
different control frequencies produce \emph{identical} coefficient
targets up to $O(1/K)$, directly resolving the identifiability
failure in Proposition~\ref{prop:identifiability}. Moreover, the frequency-domain coefficients serve as weights over cosine basis functions, enabling reconstruction of a continuous action trajectory executable at arbitrary temporal resolutions:
\begin{equation}
  \hat{v}(\tau) = \tfrac{1}{2}\hat{c}_0
  + \sum_{j=1}^{M}\hat{c}_j\cos(\omega_j\tau),
  \qquad
  \omega_j = \frac{j\pi}{T_\xi},
  \quad \tau\in[0,T_\xi].
  \label{eq:dct_continuous}
\end{equation}
At inference, the ODE is integrated in the $(M+1){\times}d$-dimensional coefficient space and the action is recovered via Eq.~\eqref{eq:dct_continuous}. \motif{} replaces the chunk-space transport target with the coefficient vector $\hat{\mathbf{c}}^*$, forming the straight-line interpolant $\hat{\mathbf{c}}_t = t\hat{\mathbf{c}}^* + (1-t)\varepsilon$, the flow-matching loss becomes
\begin{equation}
  \lfm(\theta)
  = \mathbb{E}_{(\mathbf{o},\hat{\mathbf{c}}^*),\,t,\,\varepsilon}
  \Bigl\|
    v_\theta\!\left(\hat{\mathbf{c}}_t,\,\mathbf{o},\,t\right)
    - \bigl(\hat{\mathbf{c}}^* - \varepsilon\bigr)
  \Bigr\|^2,
  \label{eq:fm_motif}
\end{equation}
which is gradient-equivalent to an $\ell^2$ error on decoded position sequences (Lemma~\ref{lem:impl_equiv}, Appendix~\ref{app:proof_h1_zihao}).

%% ------------------------------------------------------------------
\subsection{Temporally Consistent Actions via First-Order Derivative Regularization}
\label{sec:method_loss}

% During execution, actions in an action chunk are executed sequentially, introducing temporal dependency between adjacent actions in the same action chunk. However, existing flow matching paradigm treat each action in a chunk independently, ignoring temporal consistency. As a consequence, errors in predicting adjacent actions can be small yet of opposite directions, creating unstable and jittery actions. In soft-body manipulation, such unstable actions can cause spillage in liquid pouring and irreversible deformation in soft-body manipulation.

During inference, adjacent actions within a chunk are executed sequentially and are therefore temporally coupled. However, most flow-matching formulations treat the actions in a chunk independently, ignoring such temporal correlations. As a result, adjacent actions may each incur only small errors but deviate in opposite directions, producing high-frequency oscillations and jitter. Such instability can be particularly costly in soft-body manipulation, causing spillage during liquid pouring and irreversible deformation during surgical tissue handling.

% While tolerable in rigid-body manipulation, jittery actions can cause irreversible .

% are not correlated, which leads .

% Standard flow matching $\lfm$ supervises only the $K$ discrete position samples in the action chunk, leaving the continuous velocity profile between samples unconstrained. High-frequency coefficient errors can therefore manifest as action jitter even when the position loss is small. While this may be tolerable in rigid-body tasks, in soft-body manipulation (e.g.\ liquid pouring and surgical tissue handling) velocity discontinuities can directly cause irreversible deformation or spillage.

Beyond supervising per-timestep action values, the first-order temporal derivative of actions provides an explicit supervision signal for how actions should evolve over time. Penalizing changes between neighbouring actions yields coupled predictions within a chunk, discourages abrupt action variations, and improves temporal consistency. In discrete action chunks, this derivative is typically approximated using finite differences between adjacent actions. However, such estimates can be noisy and implicitly assume a piecewise-linear evolution between samples. In contrast, our DCT-based continuous action parameterization yields analytic temporal derivatives directly from the DCT coefficients, enabling accurate and stable derivative supervision. In particular,
\begin{equation}
  \hat{\dot{v}}(\tau)
  = -\sum_{j=1}^{M}\hat{c}_j\,\omega_j\sin(\omega_j\tau),
  \label{eq:vel_analytic}
\end{equation}
which is exact within the $M$-mode subspace and free of finite-difference noise. We therefore additionally supervise the first-order derivative of actions by flow-matching:
\begin{equation}
  \lvel(\theta)
  = \mathbb{E}_{(\xi,\tau_k)}
  \bigl\|\hat{\dot{v}}(\tau_k) - \dot{\xi}(\tau_k)\bigr\|^2,
  \label{eq:lvel}
\end{equation}
where $\dot{\xi}(\tau_k)$ is the first-order derivative of the demonstration. The total objective is
\begin{equation}
  \mathcal{L}_{\motif{}} = \lfm + \lambda\cdot\lvel,
  \qquad \lambda=1.
  \label{eq:total}
\end{equation}
Empirically, we find that $\lambda=1$ already yields good result across tasks and use this value throughout our paper. Accordingly, FAFM introduces no additional networks and can be applied off-the-shelf.

\textbf{Understanding the loss.} We further find our FAFM loss corresponds to squared $H^1_\mu$ Sobolev norm, and thus enhancing smoothness by penalizing high-frequency error, empirically speeds up convergence and increases robustness against mechanical bias. Specifically, let $\delta_j := \hat{c}_{\theta,j} - c_j^*$ denote coefficient prediction error. Then Lemma~\ref{lem:impl_equiv} (Appendix~\ref{app:proof_h1_zihao}) gives $\lfm = \sum_{j=0}^{M}\delta_j^2$, while DST orthogonality gives $\lvel = \sum_{j=1}^{M}\omega_j^2\,\delta_j^2$. Hence,
\begin{equation}
  \mathcal{L}_{\motif{}}(\theta)
  \;=\;
  \sum_{j=0}^{M}\mu_j\,\delta_j^2,
  \qquad
  \mu_j \;:=\; 1+\lambda\omega_j^2,
  \label{eq:h1_norm}
\end{equation}
which is the squared $H^1_\mu$ Sobolev norm on the projection residual defined below $\hat{\xi}_\theta - P_M\xi^*$. Hence,

\begin{theorem}[\motif{} loss as a weighted $H^1$-projection error]
\label{thm:h1_zihao}
Let $V_M:=\mathrm{span}\{\phi_0,\dots,\phi_M\}
\subset H^1([0,T_\xi];\mathbb{R}^d)$,
let $P_M:L^2\to V_M$ denote $L^2$-orthogonal projection,
and let
$\hat{\xi}_\theta=\sum_{j=0}^{M}\hat{c}_{\theta,j}\phi_j\in V_M$
be the \motif{}-decoded prediction.
Under centered-grid sampling $\tau_n=(n+\tfrac{1}{2})/f_\xi$
and finite-difference velocity targets,
\begin{equation}
    \mathcal{L}_{\motif{}}(\theta)
    \;=\;
    \bigl\|\hat{\xi}_\theta - P_M\xi^*\bigr\|_{H^1_\mu}^2,
    \qquad
    \|f\|_{H^1_\mu}^2
    \;:=\;
    \sum_{j\geq 0}
    (1+\lambda\omega_j^2)\,
    \bigl|\langle f,\phi_j\rangle_{L^2}\bigr|^2.
    \label{eq:loss_as_proj}
\end{equation}
Proof sketch: Centered-grid DCT orthogonality gives $\lfm=\sum_j\delta_j^2$, while DST orthogonality gives $\lvel=\sum_j\omega_j^2\delta_j^2$. Eq.~\eqref{eq:h1_norm} then follows immediately; the full proof is in Appendix~\ref{app:proof_h1_zihao}.
\end{theorem}

\textbf{Remark 1: Higher-order dynamics correction.}
The $\omega_j^2$ weight in $\mu_j$ supplies stronger gradient pressure on high-frequency coefficient errors than standard $L^2$ flow matching, directly suppressing high-frequency artifacts such as action jitter. A formal mode-wise efficiency analysis across derivative orders is given in Appendix~\ref{app:correction_efficiency}.

\textbf{Remark 2: Convergence acceleration.} For physically smooth trajectories, $\mu_j$ acts as a built-in spectral preconditioner: it encourage higher gradient on high-frequency component, which empirically accelerate convergence. A formal analysis is available in Appendix~\ref{app:convergence}.

\textbf{Remark 3: Robustness to embodiment mechanical bias.} A constant mechanical offset perturbs only the DC coefficient under \motif{}'s representation, leaving all shape-defining coefficients exactly unchanged. In contrast, a constant offset perturb all actions in flow matching chunks, causing greater harm to the robustness of flow matching policies. See Appendix~\ref{app:bias_robustness} for a derivation.

\vspace{-0.1in}
\section{Experiments}
\label{sec:exp}

\textbf{Evaluation environments}: We evaluate \motif{} in two stages. First, as a standalone policy, it is assessed on synthetic toy benchmark, obstacle avoidance, LapGym \citep{scheikl2023lapgym}, and deployed on a real Franka robot. Second, \motif{} is integrated into VLA backbones, evaluated on LIBERO \citep{liu2023libero}, and again deployed on a real Franka robot. The specific purpose of each environment is described below.

\textbf{Evaluation metrics}: We adopt success rate (SR) as the primary measure of task reliability. Motion smoothness is assessed using the log dimensionless jerk (LDLJ), computed over successful episodes, which is widely used in robot-assisted surgery and surgical skill assessment~\citep{singh2021motion, aghazadeh2025new, rodriguez2026open}. LDLJ is defined on a logarithmic scale, so small numerical differences reflect exponential improvement in the underlying motion smoothness. For both metrics, higher values indicate better performance.

\textbf{Baselines}: In the standalone-policy setting, our baselines include Diffusion Policy (DP)~\citep{chi2025diffusion}, Flow Matching Policy (FM)~\citep{zhang2024affordance}, Streaming Flow Policy (SFP)~\citep{jiang2025streaming}, FreqPolicy~\citep{zhong2025freqpolicy}, and MPD~\citep{scheikl2024movement}, with our method implemented on flow matching backbone. In the VLA setting, we use $\pi_0$~\citep{black2024pi_0} or $\pi_{0.5}$~\citep{intelligence2025pi_} as the backbone, our method is implemented by changing their action head to our \motif{}-based variant. See implementation details and additional training details in Appendix \ref{app:training_setting}.

%% -------------------------------------------------------
\subsection{Standalone Policy Experiments}
\label{sec:exp_small}

\subsubsection{Synthetic Toy Benchmark Results}

We first verify the ability to learn smooth, multimodal actions on a simple 1D task. The demonstrations comprise two crossing sinusoidal modes, as shown in Fig.~\ref{fig:toy}. While our FAFM solves this task easily, existing methods struggle on this simple setting. FM preserves both modes but yields jittery trajectories. SFP fails to separate the two modes because its policy conditions on only a single previous timestep, which is insufficient when the modes overlap. FreqPolicy exhibits higher bias and occasionally produces trajectory with lower amplitude, as its low-frequency anchor is entangled with observation features in latent space. MPD captures global trajectory structure but over-smooths local modes, resulting in lower amplitude output. In contrast, FAFM cleanly separates the two sinusoidal modes while producing smooth trajectories consistently.

\begin{figure}[!t]
    \centering
    \vspace{-0.35in}
    \includegraphics[width=\linewidth]{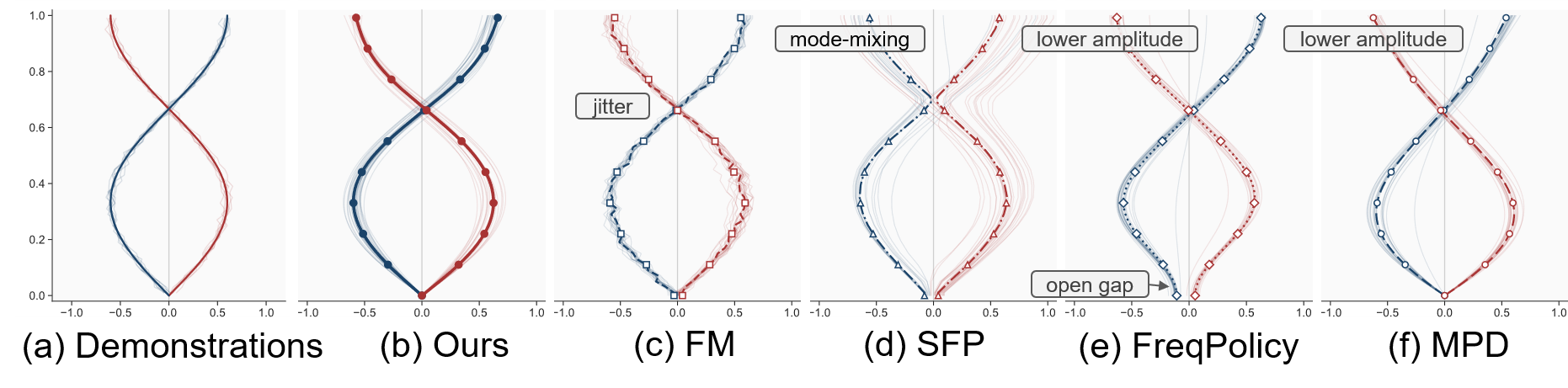}
    % \caption{Toy bi-modal trajectory experiment. (a) Training demonstrations contain two crossing S-curve modes. (b) FM in chunk space $\mathbb{R}^K$ produces high-frequency jitter and non-smooth trajectories, reflecting instability in the position-space velocity field. (c) SFP tends toward the marginal distribution of demonstrations, generating trajectories that hug the boundary region rather than committing to either mode. (d) MPD produces smooth and well-formed trajectories via RBF basis functions, but may generates mixed outputs intermediate between the two modes. (e) \motif{} produces smooth, continuous, and converged trajectories with clear separation into two distinct modes.}
    \vspace{-0.3in}
    \caption{Synthetic toy bi-modal trajectory experiment. FAFM uniquely capture the multimodal modes and yields smooth trajectory.}
    \label{fig:toy}
    \vspace{-0.2in}
\end{figure}

\begin{figure*}[!t]
\centering
\begin{minipage}[t]{0.32\textwidth}
\vspace{0.15in}
\centering
\scriptsize
\setlength{\tabcolsep}{2.5pt}
\renewcommand{\arraystretch}{0.95}
\begin{tabular}{@{}lccc@{}}
\hline
Method & SR & LDLJ & M \\
\hline
DP         & 35 & $-9.16\pm0.77$ & 8  \\
SFP        & 49 & $-6.98\pm0.82$ & 3  \\
MPD        & 16 & $-6.78\pm0.47$ & 2  \\
FreqPolicy & 39 & $-9.02\pm1.11$ & 10 \\
FM         & 48 & $-8.62\pm0.69$ & 14 \\
\hline
ours       & \textbf{61} & \textbf{-5.60$\pm$1.08} & 12 \\
\hline
\end{tabular}
\vspace{-0.1in}
\captionof{table}{Avoidance results.}
\vspace{0.1in}
\label{tab:obstacle}
\end{minipage}\hspace{0.01\textwidth}
\begin{minipage}[t]{0.64\textwidth}
\centering
\vspace{0.1in}
\includegraphics[width=\linewidth]{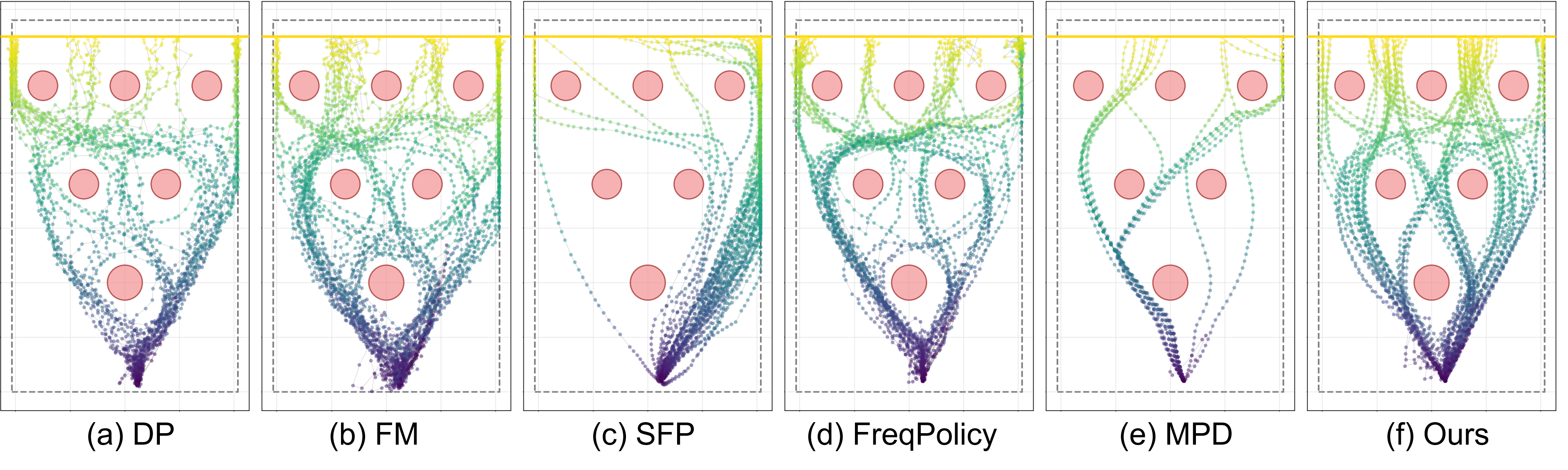}
\vspace{-0.3in}
\captionof{figure}{Successful obstacle avoidance trajectories.}
\label{fig:obstacle}
\end{minipage}
\vspace{-0.2in}
\end{figure*}

%% ============================================================
%% 4.1.2  Obstacle Avoidance: Success Rate and Multimodal Coverage
%% ============================================================

\subsubsection{Obstacle Avoidance Results}

\textbf{Settings.} Built on PyBullet, the obstacle-avoidance task requires a robot arm to navigate across multiple obstacles. Because many valid paths exist, it is a standard benchmark for evaluating solution diversity. In addition to success rate and smoothness, we report the number of distinct trajectory modes (M), defined as modes traversed by more than two successful trajectories.

\textbf{Results.} As shown in Table~\ref{tab:obstacle}, \motif{} simultaneously achieves high success rate, motion smoothness, and solution diversity. SFP and MPD produce smooth trajectories but fail to learn multimodal policies, resulting in limited path diversity. Conversely, DP, FreqPolicy, and FM discover diverse solutions but generate jittery actions, as reflected by lower LDLJ. These observations are corroborated qualitatively in Fig.~\ref{fig:obstacle}: the solutions of SFP and MPD are confined to a narrow set of paths, whereas those of DP, FM, and FreqPolicy are noisy and unstable. In contrast, FAFM is the only method achieving both smoothness and diversity.

\subsubsection{LapGym: Surgical Manipulation Results} 

\textbf{Settings.} LapGym is an open-source simulation environment for robot-assisted laparoscopic surgery. In LapGym, the robot operate on deformable objects, such as tissues and suture threads, demanding gentle and precise manipulation. As shown in Fig.~\ref{fig:lapgym}(a), we evaluate on four tasks from this environment: rope threading (\textsc{RT}), grasp-lift-touch (\textsc{GLT}), bimanual tissue manipulation (\textsc{BTM}), and ligating loop (\textsc{LL}). All methods are trained for 3,000 epochs for their best performance.

\textbf{Results.} As shown in Table~\ref{tab:lapgym_3k}, \motif{} achieves the highest success rate and motion smoothness across all four tasks. Although MPD also yields strong results, its smoothness relies on hand-specified ProDMP priors, whereas \motif{} requires no such prior. Furthermore, as shown in Fig.~\ref{fig:lapgym}(b), \motif{} converges faster than all baselines, which is a practically important advantage in time-constrained settings such as real-world RL.

\begin{table*}[t]
\centering
\caption{LapGym results.
LDLJ is reported as mean $\pm$ std. \textbf{Bold}: best per column.}
\label{tab:lapgym_3k}

\scriptsize
\setlength{\tabcolsep}{3pt}
\renewcommand{\arraystretch}{0.95}

\begin{tabular*}{\textwidth}{@{\extracolsep{\fill}}lcccccccc@{}}
\hline
& \multicolumn{2}{c}{\textsc{RT}}
& \multicolumn{2}{c}{\shortstack{\textsc{GLT}}}
& \multicolumn{2}{c}{\shortstack{\textsc{BTM}}}
& \multicolumn{2}{c}{\textsc{LL}} \\
\hline
Method
  & SR & LDLJ
  & SR & LDLJ
  & SR & LDLJ
  & SR & LDLJ \\
\hline
DP
  & 92 & $-10.88\pm1.61$
  & \textbf{100} & $-16.62\pm0.52$
  & 96 & $-6.05\pm2.18$
  & \textbf{100} & $-13.04\pm0.93$ \\
SFP
  & 72 & $-11.54\pm2.26$
  & 13 & $-17.09\pm1.44$
  & 95 & $-5.62\pm2.53$
  & 99 & $-12.46\pm1.33$ \\
MPD
  & 94 & $-9.38\pm2.12$
  & 99 & $-14.37\pm0.52$
  & \textbf{100} & $-3.81\pm1.92$
  & \textbf{100} & $-11.47\pm1.00$ \\
FreqPolicy
  & 89 & $-9.33\pm1.40$
  & \textbf{100} & $-15.52\pm0.55$
  & 83 & $-4.71\pm1.88$
  & \textbf{100} & $-12.49\pm0.90$ \\
FM
  & 89 & $-10.69\pm2.06$
  & \textbf{100} & $-16.62\pm0.60$
  & 94 & $-5.82\pm1.87$
  & \textbf{100} & $-13.23\pm0.87$ \\
\hline
\motif{}
  & \textbf{97} & $\mathbf{-7.57\pm1.32}$
  & \textbf{100} & $\mathbf{-14.31\pm0.58}$
  & 99 & $\mathbf{-2.21\pm2.34}$
  & \textbf{100} & $\mathbf{-11.28\pm1.17}$ \\
\hline
\end{tabular*}
\end{table*}

\begin{figure*}[t]
\centering
\vspace{-0.1in}
\includegraphics[width=\textwidth]{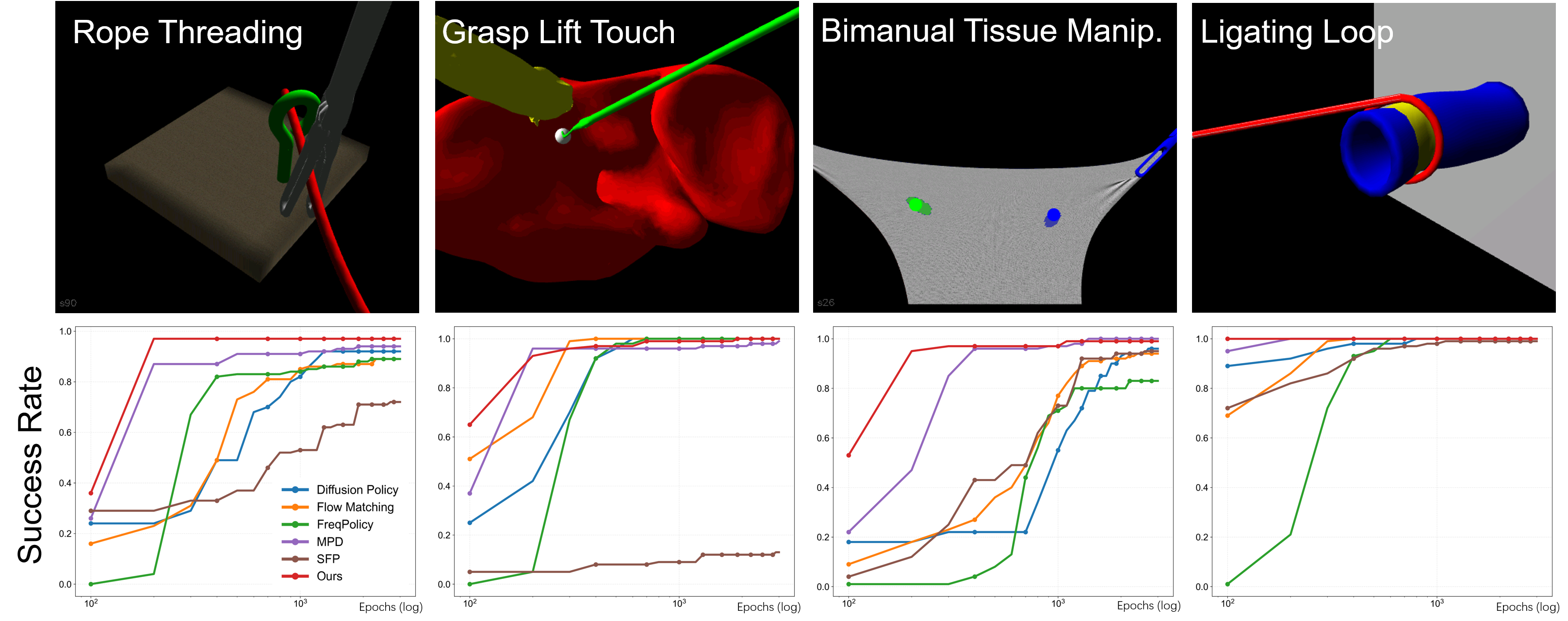}
\vspace{-0.2in}
\caption{\textbf{Top row}: Overview of surgical manipulation tasks in LapGym. \textbf{Bottom row}: Success rate on LapGym tasks as training proceeds. Our \motif{} converge to the best result earlier than baselines.}
\label{fig:lapgym}
\vspace{-0.2in}
\end{figure*}

\subsubsection{Ablation Study} 
\label{sec:ablation}

% Additionally, we evaluate the effectiveness of each component of \motif{} on rope threading. For ours w/o DCT, it achieves similar success rate and LDLJ comparing with flow matching policy, indicating frequency domain modelling contributes the most to the overall performance. Surprisingly, ours w/o DCT exhibit negligible improvement on LDLJ, indicating supervising first-order derivative by finite difference is not effective without continuous frequency domain parameterization. For ours w/o $\lvel$, it shows a moderate decrease in success rate and LDLJ, showing the effectiveness of supervising first-order derivative as $H^1_\mu$ Sobolev normalization.

We further evaluate the contribution of each component of \motif{} on the rope-threading task as shown in Fig.~\ref{tab:ablation}. Removing the DCT transform (ours w/o DCT) reduce success rates and LDLJ to flow-matching backbone, indicating that frequency-domain modelling is the dominant contributor to overall performance. Notably, ours w/o DCT also shows negligible improvement in LDLJ, suggesting that supervising the first-order derivative via finite differences is ineffective without the continuous frequency-domain parameterization. Removing the derivative supervision (ours w/o $\lvel$) results in a moderate drop in both success rate and LDLJ, confirming the benefit of supervising the first-order derivative as $H^1_\mu$ Sobolev regularization.

% We ablate \motif{} on 's two components on rope threading: DCT parameterization and $\lvel$ (Figure~\ref{tab:ablation}). DCT alone improves SR from $89\%$ to $95\%$ by resolving the step-index identifiability failure; adding $\lvel$ further reduces jerk (LDLJ $-9.12\to-7.57$), consistent with Theorem~\ref{thm:h1_zihao}. Applying $\lvel$ without DCT yields negligible gain, confirming that velocity supervision requires the frequency-domain structure to be informative.

\begin{figure*}[t]
\centering
\begin{minipage}[t]{0.44\textwidth}
\centering
\includegraphics[width=\linewidth]{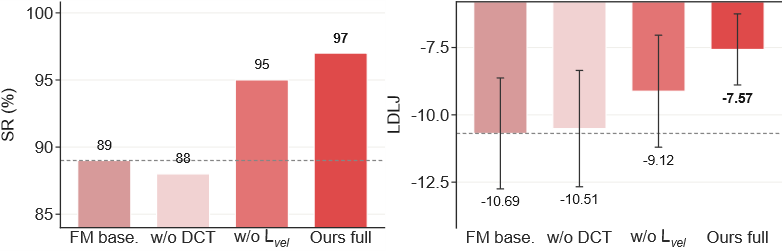}
\vspace{-0.2in}
\captionof{figure}{Ablation on rope threading.}
\label{tab:ablation}
\end{minipage}\hfill
\begin{minipage}[t]{0.55\textwidth}
\centering
\includegraphics[width=\linewidth]{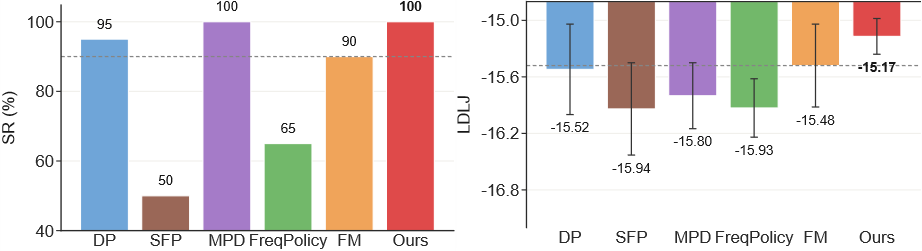}
\vspace{-0.2in}
\captionof{figure}{Real-robot policy pick\&place.}
\label{tab:real_small}
\end{minipage}
\end{figure*}

\subsubsection{Real-Robot Evaluation}
\label{sec:exp_real_small}

Finally, we deploy all methods on a Franka physical robot arm for pick-and-place, which requires picking a cup and place it on the plate. For all methods, we fine-tune it on 40 demos and evaluate the result over 20 episodes. As shown in Figure~\ref{tab:real_small}, While our \motif{} and MPD both achieves $100\%$ success rate, out FAFM achieve better motion smoothness measured by LDLJ. It is also worth noting that SFP drops to $50\%$ SR on the physical system, likely due to its marginal-only trajectory modeling being sensitive to real-world observation noise.

% We deploy all methods on a physical robot arm for a pick-and-place
% task.
% SPARC and LDLJ are computed from recorded end-effector trajectories
% (Table~\ref{tab:real_small}).

% \begin{table}[h]
% \centering
% \caption{Real-robot pick-and-place (small model).}
% \label{tab:real_small}
% \begin{tabular}{lccc}
% \toprule
% Method & SR (\%) & SPARC & LDLJ \\
% \midrule
% DP          & 95 & $-7.40$ & $-15.52$ \\
% SFP         & 50 & $-7.56$ & $-15.94$ \\
% MPD         & \textbf{100} & $-7.17$ & $-15.80$ \\
% FM          & 90 & $-8.02$ & $-15.48$ \\
% \midrule
% \motif{}    & 95 & \textbf{$-7.59$} & \textbf{$-15.39$} \\
% \bottomrule
% \end{tabular}
% \end{table}

% On the real robot, \motif{} matches the top task-success methods
% while achieving the best motion quality among methods with SR $\geq
% 90\%$.
% MPD achieves $100\%$ SR but at a cost of higher jerk ($-15.80$),
% consistent with its training objective being disconnected from
% velocity quality.
% SFP's SR drops to $50\%$ on the physical system, likely due to its
% marginal-only trajectory modeling being sensitive to observation
% noise at execution time.

%% ============================================================
%% 4.2  Large-Model Experiments (pi0 scale)
%% ============================================================
\subsection{VLA Scaling Experiments}
\label{sec:exp_large}

%% ============================================================
%% 4.2.1  LIBERO Robustness
%% ============================================================
\subsubsection{LIBERO Results}
\label{sec:exp_libero_robust}

\textbf{FAFM enhance smoothness and is robust against mechanical bias.} We evaluate the performance of \motif{} on LIBERO, a commonly used benchmark for VLAs, implementing our method on $\pi_0$ and $\pi_{0.5}$ backbone and treating the unmodified VLAs as baselines. As shown in Figure~\ref{tab:libero}, \motif{} achieves comparable success rates on both backbones while consistently improving motion smoothness. Under mechanical bias, \motif{} yields substantially higher success rates alongside improved smoothness on both backbones, confirming that our DCT-based action parameterization is robust to constant action bias.

\textbf{FAFM supports mixed-frequency input.} Training VLAs typically requires large datasets collected across diverse hardware platforms, which are often recorded at heterogeneous control frequencies. To evaluate robustness to this setting, we conduct a controlled experiment on the LIBERO drawer task and real-robot pick-and-place task with 43 mixed-frequency demonstrations at $5$/$10$/$20\,\mathrm{Hz}$.\motif{} fixes $T_\xi=5\,\mathrm{s}$, mapping $K=25/50/100$ samples to $M+1=17$ DCT coefficients. We compare against $\pi_0$-Resample, which linearly resamples demonstrations to $10\,\mathrm{Hz}$, and $\pi_0$-FreqCond, which conditions on recording frequency as in RDT-1B~\cite{liu2024rdt}. As shown in Fig.~\ref{tab:mixed_fps}, \motif{} achieves higher success rates than both baselines in simulation and real-world experiments, with improved motion smoothness reported in Appendix~\ref{app:additional_experiments}.

\begin{figure*}[t]
\centering
\begin{minipage}[t]{0.6\textwidth}
\centering
\includegraphics[height=2.4cm]{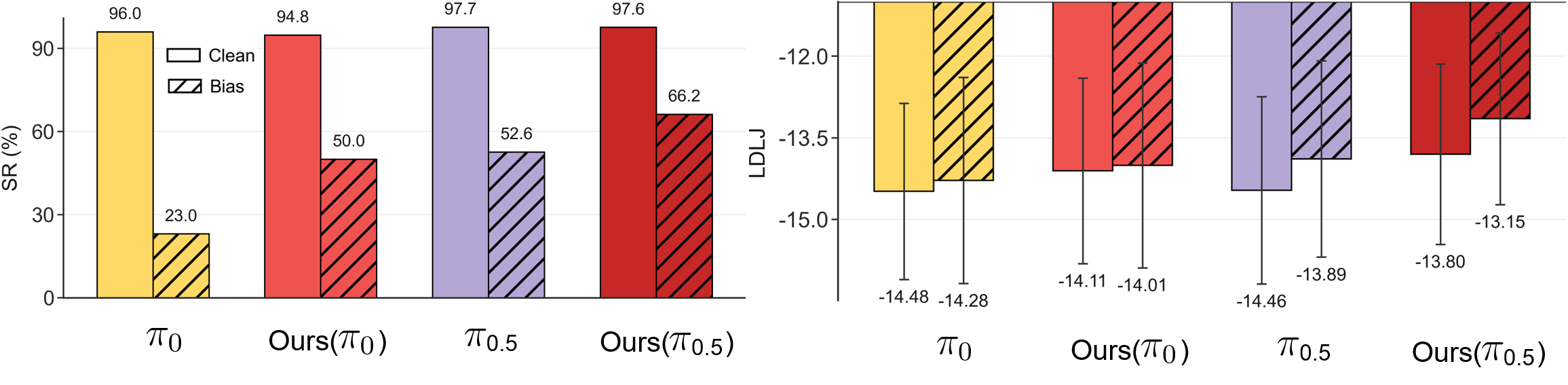}
\vspace{-0.2in}
\captionof{figure}{LIBERO results. ``Clean'' is the unperturbed baseline. Bias denotes mechanical offset.}
\vspace{-0.1in}
\label{tab:libero}
\end{minipage}
\vspace{-0.1in}
\hfill
\begin{minipage}[t]{0.3\textwidth}
\centering
\includegraphics[height=2.4cm]{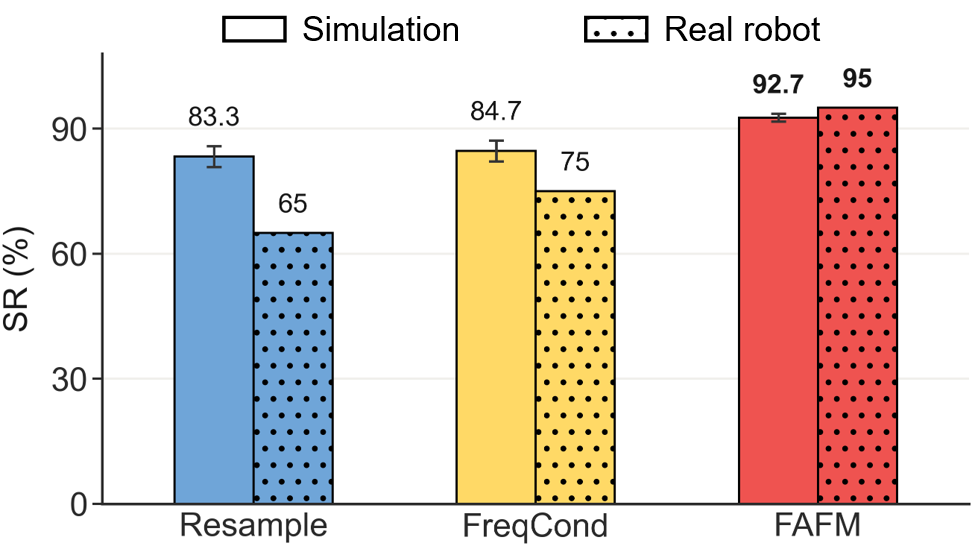}
\vspace{-0.1in}
\captionof{figure}{mixed-freq results.}
\label{tab:mixed_fps}
\end{minipage}
\end{figure*}

\subsection{Real-Robot Deployment with $\pi_{0.5}$}
\label{sec:exp_real_large}

\begin{figure*}[!t]
\centering
% \vspace{in}
\includegraphics[width=0.9\textwidth]{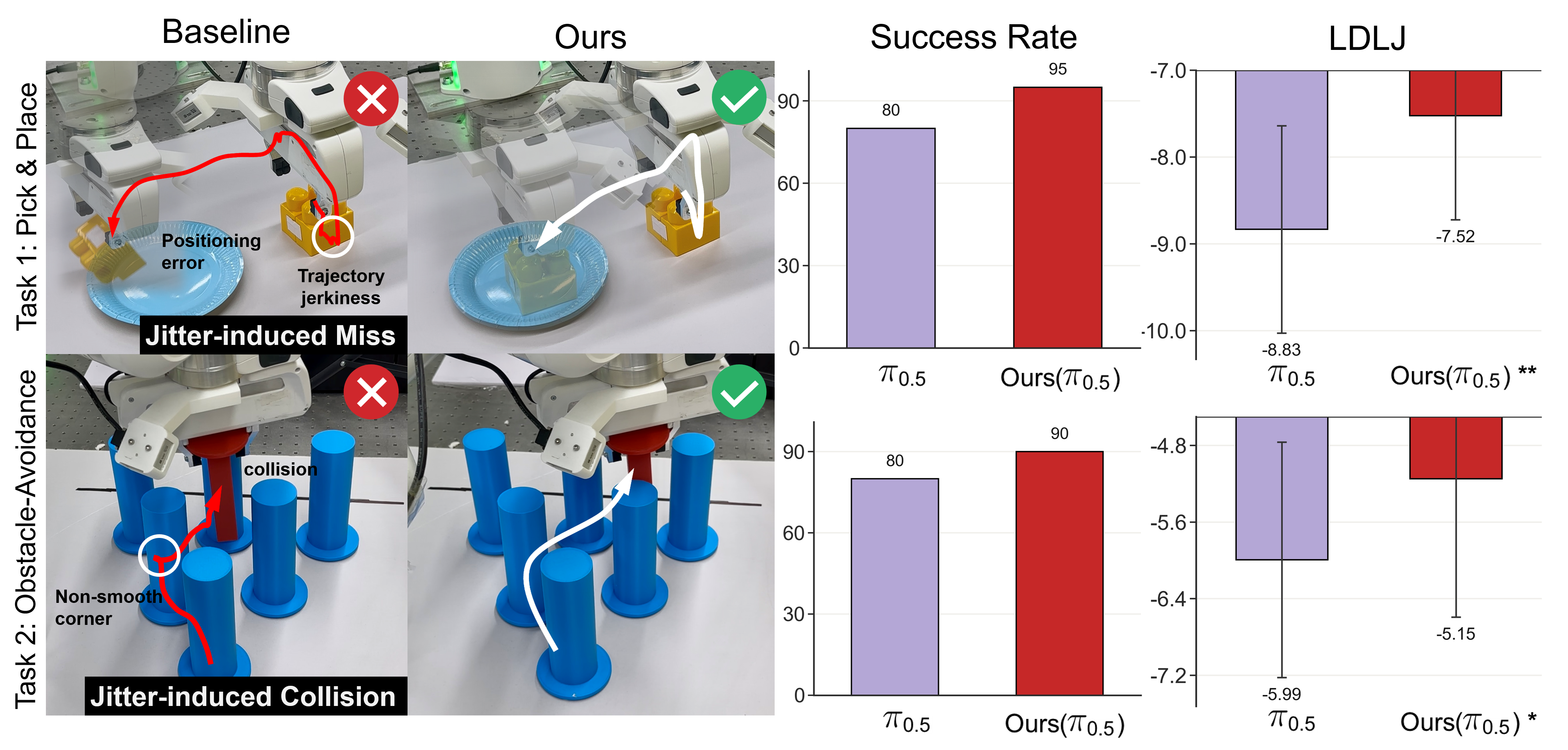}
\vspace{-0.05in}
\caption{Real-world deployment with the $\pi_{0.5}$ backbone. Our FAFM achieves higher success rate by improved motion smoothness (LDLJ), which avoids jitter-induced miss and collisions. For LDLJ, paired samples t-test shows $p<0.001$$ (^{**}$) for Task 1, $p<0.05$ ($^{*}$) for Task 2.}
\label{fig:realworld}
\vspace{-0.1in}
\end{figure*}

Finally, we evaluate the performance of our method on $\pi_{0.5}$ backbone. We include two tasks for comparison. The first one is pick-and-place with variable object positions and goal configurations, and the second is multi-obstacle avoidance with slightly perturbed obstacle positions and start points, see Fig.~\ref{fig:realworld} for illustrations. We collect 20 demos for pick-and-place and 60 demos for multi-obstacle avoidance, which is sufficient for their convergence, and evaluate each method and each task for 20 episodes. Our method achieves $+15\%$ success rate on pick-and-place and $+10\%$ success rate on obstacle avoidance comparing with $\pi_{0.5}$ backbone. The higher success rate comes from improved motion smoothness, which better avoids jitter-induced target missing and collisions in both tasks.

\section{Conclusion}
\label{sec:conclusion}

We presented FAFM, a frequency‑aware flow‑matching framework that enables continuous and temporally consistent action generation for robotic manipulation. By performing flow matching in frequency space via DCT and regularizing the first-order derivative, our method supports mixed-frequency training and provides smooth actions. Our FAFM is simple, free of additional learnable parameters, and applies to standalone flow matching policy and VLAs. Experimental results across synthetic toy benchmarks, obstacle avoidance, LapGym, and LIBERO demonstrate consistent improvements over strong baselines in success rates, multimodal expressivity, motion smoothness, convergence speed, robustness to mechanical bias and mixed-frequency input. These advantages are further validated on real‑world Franka robot. As for \textbf{limitations}, our approach is less effective for tasks that intrinsically require high‑frequency or impulsive actions due to the nature of DCT, such as sharp jumps or percussive interactions. However, the vast majority of real‑world robotic manipulation tasks are characterized by smooth, low‑frequency action trajectories suitable for our FAFM.

% \newpage
%% ============================================================
\begin{ack}
This work described in this paper was supported in part by a grant from the NSFC/RGC Joint Research Scheme sponsored by the Research Grants Council of the Hong Kong Special Administrative Region, China and the National Natural Science Foundation of China (Project No. N\_CUHK410/23) and by National Natural Science Foundation of China (Project No. 62606026, 62322318). This work was also supported by Beijing Natural Science Foundation (Project No. QY26142).
\end{ack}

\bibliographystyle{unsrt}
\bibliography{neurips_2026}

\newpage
\appendix

% \section{Frequency-Invariance of DCT Coefficients}
% \label{app:freq_inv}
 
% \begin{proposition}[DCT frequency invariance]
% \label{prop:freq_inv}
% Let two demonstrations execute the same physical motion $\xi^*$ over
% $[0,T_\xi]$, sampled at $f_s\neq f_f$ with $K_s=\lfloor T_\xi f_s\rfloor$,
% $K_f=\lfloor T_\xi f_f\rfloor$.
% For all $j\leq\min(K_s,K_f)/2$:
% \[
%   \hat{c}_j^{(f_s)}=c_j^*+O(1/K_s),\qquad
%   \hat{c}_j^{(f_f)}=c_j^*+O(1/K_f).
% \]
% \end{proposition}
% \begin{proof}
% Eq.~\eqref{eq:dct2_forward} is a midpoint quadrature rule for $c_j^*$
% with step $\Delta\tau=1/f_\xi$.
% By the Euler--Maclaurin formula, the quadrature error is
% $O(\Delta\tau)=O(1/K)$ for Lipschitz $\xi^*$, and $O(1/K^2)$ for $C^2$
% trajectories.\quad$\square$
% \end{proof}
 
% \noindent With $f_\xi\geq 10\,\mathrm{Hz}$, $T_\xi$ fixed, and $M=16$,
% the Nyquist condition $M<K/2$ holds automatically and the $O(1/K)$
% error is numerically negligible.
 
%% ------------------------------------------------------------------
\section{Proof of Proposition~\ref{prop:identifiability}}
\label{app:proof_identifiability}

\begin{proof}
    % L2 Bayes optimum is the conditional mean
	Fix $(\mathbf{o}, k)$ and write $Y = \xi(k/f) \cdot \mathbf{1}\{k < K(\xi, f)\}$ for the regression target restricted to the event that token $k$ exists. The conditional expectation $\mathbb{E}[Y \mid \mathbf{o}, k]$ is the $L^2$-projection of $Y$ onto the $\sigma$-algebra generated by $(\mathbf{o}, k)$, so for any measurable predictor $\hat{a}(\mathbf{o}, k)$,
	\begin{equation}\notag
	\mathbb{E}\bigl[(\hat{a} - Y)^2 \mid \mathbf{o}, k\bigr]
	\;=\;
	\bigl(\hat{a}(\mathbf{o},k) - \mathbb{E}[Y \mid \mathbf{o}, k]\bigr)^{2}
	\;+\;
	\mathrm{Var}(Y \mid \mathbf{o}, k),
	\end{equation}
	in which only the first term depends on $\hat{a}$. Minimization gives $\hat{a}^{\star}(\mathbf{o}, k) = \mathbb{E}[Y \mid \mathbf{o}, k]$, which is \eqref{eq:bayes-optimum}.

    Under the non-degeneracy assumption, the integrand $\xi(k/f)$ in \eqref{eq:bayes-optimum} takes values $\xi(k/f_1), \xi(k/f_2), \ldots$ at \emph{different} physical times $k/f_1 \neq k/f_2$ across demonstrations. The expectation therefore averages evaluations of $\xi$ at \emph{distinct phases of the underlying motion}, not at the same phase of different motions. A convex combination of points sampled from distinct phases of even a single trajectory $\xi$ generally falls off the trajectory itself; a fortiori the same holds when the points are drawn from different trajectories at different phases.

\end{proof}

\begin{remark}
The predictor $\hat{a}(k \mid \mathbf{o})$ above depends on $(\mathbf{o}, k)$ but don't have the information about $f$. 
\end{remark}

\begin{remark}
The function $\hat{a}^{\star}$ in \eqref{eq:bayes-optimum} depends only on the conditioning $(\mathbf{o}, k)$ and the distribution $p_{\mathcal{D}}$; it is independent of training-set size, model capacity, and optimization procedure.Because the right-hand side of \eqref{eq:bayes-optimum} is a deterministic functional of $p_{\mathcal{D}}$ alone. Increasing the training set drives the empirical risk to its population infimum $\mathcal{J}(\hat{a}^{\star})$; increasing model capacity makes the empirical minimizer closer to $\hat{a}^{\star}$ in $L^{2}$; neither changes $\hat{a}^{\star}$ itself. The optimization procedure determines how closely the trained model approximates $\hat{a}^{\star}$, but cannot reach a target lower than $\mathcal{J}(\hat{a}^{\star})$.
\end{remark}

\paragraph{Extension to Flow Matching.}
For straight-interpolation conditional flow matching, let $A_t=tA+(1-t)\varepsilon$, $Z=(A_t,o,t)$, and $U=A-\varepsilon$. The FM objective is
\begin{equation}
\mathcal{L}_{\mathrm{FM}}
=
\mathbb{E}\left[\|v_\theta(Z)-U\|^2\right],
\end{equation}

whose optimum is $v^*(Z)=\mathbb{E}[U\mid Z]$. When the recording frequency $F$ is unobserved, the law of total variance gives

\begin{equation}
% \begin{aligned}
\mathrm{Var}(U\mid Z)
={}
\mathbb{E}[\mathrm{Var}(U\mid Z,F)\mid Z]+
\mathrm{Var}(\mathbb{E}[U\mid Z,F]\mid Z).
% \end{aligned}
\end{equation}

The second term captures additional between-frequency target variance. Thus, frequency-agnostic step-indexed flow matching can introduce conflicting gradient signals when targets from different recording frequencies overlap.

\section{Proof of Theorem~\ref{thm:h1_zihao}}
\label{app:proof_h1_zihao}

% We model the physical trajectory as an element of $\mathcal{H}:=H^1([0,T_\xi];\mathbb{R}^d) :=\{f\in L^2(0,T_\xi;\mathbb{R}^d)\mid f'\in L^2(0,T_\xi;\mathbb{R}^d)\}$ with inner product
% \begin{equation}
%     \langle f,g\rangle_{H^1}
%     \;=\;
%     \langle f,g\rangle_{L^2}
%     \;+\;
%     \langle f',g'\rangle_{L^2},
%     \label{eq:h1_inner}
% \end{equation}
% and expand in the $L^2$-orthonormal cosine basis $\phi_j(\tau)=\alpha_j\cos(\omega_j\tau)$,
% $\omega_j=j\pi/T_\xi$, $\alpha_0=\sqrt{1/T_\xi}$, $\alpha_j=\sqrt{2/T_\xi}$ for $j\geq 1$, which diagonalizes $-\partial_\tau^2$ under Neumann boundary conditions. The target trajectory admits the expansion $\xi^*=\sum_{j\geq 0}c_j^*\phi_j$ with $c_j^*=\langle\xi^*,\phi_j\rangle_{L^2}$.

Before proving Theorem~\ref{thm:h1_zihao}, we establish that the position-space implementation of $\lfm$ used in practice is gradient-equivalent to the coefficient-space form in Eq.~\eqref{eq:fm_motif}.
 
\begin{lemma}[Implementation equivalence]
\label{lem:impl_equiv}
Let $\hat{\mathbf{c}}_\theta \in \mathbb{R}^{M+1}$ be the coefficient-space prediction and $\hat{\xi}_\theta(\tau_n) = \sum_{j=0}^{M}\hat{c}_{\theta,j}\phi_j(\tau_n)$ its decoded trajectory at the centered grid $\tau_n=(n+\tfrac{1}{2})/f_\xi$. Let $\hat{\mathbf{c}}^*\in\mathbb{R}^{M+1}$ be the target coefficients and $\xi^*_n := \xi^*(\tau_n)$ the target chunk. Then
\begin{equation}
    \sum_{n=0}^{K-1}\bigl(\hat{\xi}_\theta(\tau_n)-\xi^*_n\bigr)^2
    \;=\;
    \frac{K}{T_\xi}\sum_{j=0}^{M}\bigl(\hat{c}_{\theta,j}-c_j^*\bigr)^2
    \;+\;
    R_M(\xi^*),
    \label{eq:impl_equiv}
\end{equation}
where $R_M(\xi^*)=\sum_{n=0}^{K-1}\bigl(\xi^*_n-(P_M\xi^*)(\tau_n)\bigr)^2$ depends only on the high-frequency tail $\{c_j^*\}_{j>M}$ of the target trajectory and not on $\theta$. 
\end{lemma}

\begin{proof}
Decompose the chunk-space residual at each sample point as
\begin{equation}
    \hat{\xi}_\theta(\tau_n) - \xi^*_n
    \;=\;
    \underbrace{\hat{\xi}_\theta(\tau_n) - (P_M\xi^*)(\tau_n)}_{\in V_M}
    \;+\;
    \underbrace{(P_M\xi^*)(\tau_n) - \xi^*_n}_{\perp V_M\text{ in }L^2},
    \label{eq:residual_split}
\end{equation}
where the first term is in $V_M$ and equals $\sum_{j=0}^{M}\delta_j\phi_j(\tau_n)$, and the second is the truncation tail $-\sum_{j>M}c_j^*\phi_j(\tau_n)$. On the centered grid, the discrete inner product $\langle f,g\rangle_K := \sum_{n=0}^{K-1}f(\tau_n)g(\tau_n)$ satisfies the discrete orthogonality $\langle\phi_i,\phi_j\rangle_K=(K/T_\xi)\boldsymbol{\delta_{ij}}$ for $0\leq i,j\leq K-1$ ($\boldsymbol{\delta_{ij}}$ is the Kronecker symbol), so squaring~\eqref{eq:residual_split} and summing over $n$ gives
\begin{equation}
    \sum_{n=0}^{K-1}\bigl(\hat{\xi}_\theta(\tau_n)-\xi^*_n\bigr)^2
    \;=\;
    \frac{K}{T_\xi}\sum_{j=0}^{M}\delta_j^2
    \;+\;
    R_M(\xi^*),
    \label{eq:impl_equiv_expanded_orn}
\end{equation}
where the cross term vanishes by orthogonality between modes $j\leq M$ (in-$V_M$ part) and modes $j>M$ (truncation tail), and $R_M(\xi^*)=\sum_{n=0}^{K-1}\bigl(\xi^*_n-(P_M\xi^*)(\tau_n)\bigr)^2$ is $\theta$-independent. 
\end{proof}

From Lemma \ref{lem:impl_equiv}, it can be noticed that the left side of equation \ref{eq:impl_equiv} is the chunk space training loss $\mathcal{L}^{\mathrm{chunk}}(\theta)$, and the coefficient-space training loss can be written as $\mathcal{L}^{\mathrm{coef}}(\theta)\;:=\;\sum_{j=0}^{M}\delta_j^2$. So equation \ref{eq:impl_equiv} can also be written as:

\begin{equation}
    \mathcal{L}^{\mathrm{chunk}}(\theta)
    \;=\;
    \frac{K}{T_\xi}\sum_{j=0}^{M}\delta_j^2
    \;+\;
    R_M(\xi^*)
    \;=\;
    \frac{K}{T_\xi}\,\mathcal{L}^{\mathrm{coef}}(\theta)
    \;+\;
    R_M(\xi^*),
    \label{eq:impl_equiv_expanded}
\end{equation}

Consequently, since $R_M(\xi^*)$ is $\theta$-independent and $K/T_\xi>0$, the two losses share the same minimizer and the same gradient direction in $\hat{\mathbf{c}}_\theta$.

Here we formally start to proof Theorem~\ref{thm:h1_zihao}.
\begin{proof}
From Lemma \ref{lem:impl_equiv} we have Equ.\ref{eq:impl_equiv_expanded}, with $R_M(\xi^*)$ independ on $\theta$. We adopt the convention that the flow-matching  loss is the coefficient-space form  $\mathcal{L}_{\mathrm{FM}}(\theta):=\mathcal{L}^{\mathrm{coef}}(\theta)$, which is equivalent to the chunk-space form up to the $\theta$-independent affine transformation in~Equ.\ref{eq:impl_equiv_expanded}. Thus
\begin{equation}
    \lfm(\theta) \;=\; \sum_{j=0}^{M}\delta_j^2,
    \label{eq:lfm_to_l2}
\end{equation}
and the same $\theta$ minimizes both $\mathcal{L}_{\mathrm{FM}}$ and $\mathcal{L}^{\mathrm{chunk}}$.

The model velocity is the analytic derivative $\hat{\dot{\xi}}_\theta(\tau)=-\sum_{j=1}^{M}\hat{c}_{\theta,j}\alpha_j\omega_j\sin(\omega_j\tau)$. Under the finite-difference target convention, the demonstration  velocity admits the matching expansion  $\dot{\xi}^*(\tau_n)= -\sum_{j=1}^{M}c_j^*\alpha_j\omega_j\sin(\omega_j\tau_n)+r_n$, where $r_n$ collects the truncation tail $j>M$ and the finite-difference discretization error. $r_n$ is  $\langle\cdot,\cdot\rangle_K$-orthogonal to  $\{\sin(\omega_j\tau_n)\}_{j=1}^{M}$ up to $O(1/K^2)$ and contributes only a $\theta$-independent additive constant to the loss, which we absorb into a remainder $R_M^{\mathrm{vel}}(\xi^*)$ analogous to $R_M(\xi^*)$ in Lemma ~\ref{lem:impl_equiv}.

The velocity-domain MSE then expands as
\begin{align}
    \sum_{n=0}^{K-1}\bigl(\hat{\dot{\xi}}_\theta(\tau_n)
    -\dot{\xi}^*(\tau_n)\bigr)^2
    &=\sum_{n=0}^{K-1}\Bigl(\sum_{j=1}^{M}
    \delta_j\alpha_j\omega_j\sin(\omega_j\tau_n)\Bigr)^{\!2}
    \;+\; R_M^{\mathrm{vel}}(\xi^*)
    \notag\\
    &=
    \sum_{i,j=1}^{M}\delta_i\delta_j\,
    \alpha_i\alpha_j\,\omega_i\omega_j
    \sum_{n=0}^{K-1}\sin(\omega_i\tau_n)\sin(\omega_j\tau_n)
    \;+\; R_M^{\mathrm{vel}}
    \notag\\
    &=
    \frac{K}{T_\xi}\sum_{j=1}^{M}\omega_j^2\,\delta_j^2
    \;+\; R_M^{\mathrm{vel}},
    \label{eq:lvel_diag}
\end{align}
Note that here by DST orthogonality, $\sum_n sin(\omega_i\cdot\tau_n)sin(\omega_j\cdot\tau_n)=(K/2)\boldsymbol{\delta_{ij}}$,   $\alpha_j^2=2/T_\xi$, which gives $\alpha_j^2\cdot K/2=K/T_\xi$ for each diagonal term.
We define the coefficient-space velocity loss 
\begin{equation}
    \lvel(\theta) \;:=\; \frac{T_\xi}{K}
    \sum_{n=0}^{K-1}\bigl(\hat{\dot{\xi}}_\theta(\tau_n)
    -\dot{\xi}^*(\tau_n)\bigr)^2
    \;-\; \frac{T_\xi}{K}R_M^{\mathrm{vel}}
    \;=\;\sum_{j=1}^{M}\omega_j^2\,\delta_j^2,
    \label{eq:lvel_to_h1}
\end{equation}
which is equivalent to the implementation-level velocity MSE up to a global rescaling and a $\theta$-independent constant.

So we have 
\begin{equation}
  \mathcal{L}_{\motif{}}(\theta)
  \;=\; \lfm(\theta) + \lambda\,\lvel(\theta)
  \;=\; \sum_{j=0}^{M}\bigl(1+\lambda\,\omega_j^2\bigr)\,\delta_j^2.
\end{equation}

\end{proof}

\paragraph{Finite-Sampling Approximation.}
We evaluate the approximation gap on 1715 Rope Threading chunks. The implemented finite-difference objective is 0.03847, compared with 0.03714 for the ideal spectral objective, corresponding to a 3.59\% deviation. The coefficient-gradient cosine similarity is $0.919\pm0.060$, and the mean gradient norms differ by only 0.8\%.

We further compare the same physical trajectories sampled at 10 Hz ($K=12$) and 5 Hz ($K=6$), keeping $T_\xi=1.2\,\mathrm{s}$ and $M=4$ fixed. The coefficient-target RMSE is 0.02685, and the reconstructed trajectory RMSE is 0.00947 (2.51\% relative).

\section{Higher-Order Dynamics Correction Efficiency}
\label{app:correction_efficiency}
Given a per-mode loss weight $w_j$, the gradient pressure on mode $j$ is $2w_j\delta_j$. The \emph{correction efficiency at order $s$} is the ratio between gradient pressure and the resulting $s$-th-derivative error amplitude:
\begin{equation}
	\eta_j^{(s)} \;:=\; \frac{w_j}{A_j^{(s)}}
	\;=\; \frac{w_j}{\omega_j^{s}}.
\end{equation}
A larger $\eta_j^{(s)}$ means stronger gradient pressure per unit of $s$-th-derivative error---the loss is more effective at correcting errors of dynamic order $s$ at frequency $\omega_j$. Comparing standard chunk-space FM ($w_j^{\mathrm{FM}}=1$) with \motif{} ($w_j^{\motif{}}=\mu_j=1+\lambda\omega_j^{2}$):
\begin{center}
	\renewcommand{\arraystretch}{1.25}
	\begin{tabular}{lcccc}
		\toprule
		& Position ($s{=}0$) & Velocity ($s{=}1$) & Acceleration ($s{=}2$) & Jerk ($s{=}3$) \\
		\midrule
		Amplification $A_j^{(s)}$ & $1$ & $\omega_j$ & $\omega_j^{2}$ & $\omega_j^{3}$ \\
		\midrule
		$\eta_j^{(s),\,\mathrm{FM}}$ & $1$ & $\propto 1/j$ & $\propto 1/j^{2}$ & $\propto 1/j^{3}$ \\
		$\eta_j^{(s),\,\motif{}}$ & $1+\lambda\omega_j^{2}$ & $\propto j$ & $\to\lambda$ & $\propto 1/j$ \\
		\midrule
		High-freq.\ ratio & $O(j^{2})$ & $O(j^{2})$ & $O(j^{2})$ & $O(j^{2})$ \\
		\bottomrule
	\end{tabular}
\end{center}
Standard FM exhibits $\eta_j^{(s)}\propto 1/j^{s}$, decaying at every order, while \motif{} replaces this decay with the high-frequency behaviors shown above. The result is a uniform $O(j^{2})$ improvement over standard FM at every derivative order, with two qualitative consequences:

\textbf{\emph{Velocity ($s{=}1$):}} \motif{} gives \emph{growing} correction efficiency $\propto j$, while standard FM \emph{decays} as $1/j$. High-frequency velocity errors are corrected aggressively under \motif{} and largely ignored under standard FM.

\textbf{\emph{Acceleration ($s{=}2$):}} \motif{}'s efficiency is asymptotically uniform, $\eta_j^{(2),\,\motif{}}\to\lambda$ as $j\to\infty$. This uniformity is not a tuned property but a direct consequence of the $H^{1}$ norm: the derivative weight $\omega_j^{2}$ in $\mu_j$ exactly matches the acceleration amplification $A_j^{(2)}=\omega_j^{2}$. No separate hyperparameter is responsible for this match.

The construction extends to higher derivative orders by analogous supervision. This stronger supervision on high-frequency coefficient errors directly suppresses high-frequency artifacts such as action jitter.
%% ------------------------------------------------------------------
\section{$H^\infty$ Natural Regularity}
\label{app:sobolev_reg}
 
\begin{proposition}[$\hat{v}\in H^\infty$]
\label{prop:hinf}
Any trajectory decoded by Eq.~\eqref{eq:dct_continuous} with finite $M$
satisfies $\hat{v}\in H^\infty[0,T_\xi]=\bigcap_{s\geq 0}H^s[0,T_\xi]$:
\[
  \|\hat{v}\|_{H^s}^2
  = \frac{T_\xi}{2}\sum_{j=1}^{M}
      \bigl(1+\omega_j^2+\cdots+\omega_j^{2s}\bigr)|\hat{c}_j|^2
  < \infty,
  \quad\forall\,s\geq 0.
\]
\end{proposition}
\begin{proof}
$M$, $\omega_j$, and $|\hat{c}_j|$ are finite; a finite sum of finite
terms is finite for any $s$.\quad$\square$
\end{proof}
 
\noindent Standard chunk FM outputs $[a_0,\ldots,a_{K-1}]$ are discrete
sequences; Sobolev regularity is not defined for them.
Theorem~\ref{thm:h1_zihao} and Proposition~\ref{prop:hinf} together show that
\motif{} operates in a Sobolev-regular output space ($H^\infty$, by
architecture) and optimizes a Sobolev-consistent objective ($H^1$, by
training).

\section{Convergence acceleration}
\label{app:convergence}
For smooth physical trajectories, energy compaction implies
$|c_j^*| \lesssim C/j^2$.
Under standard $L^2$ flow matching, the per-mode gradient
$2\delta_j$ vanishes for large $j$, causing slow convergence for
fine dynamic details.
By Theorem~\ref{thm:h1_zihao}, the weight $\mu_j\sim O(j^2)$
acts as a spectral preconditioner that exactly counteracts this
$O(1/j^2)$ decay, equalizing gradient pressure across all
frequency bands and preventing high-frequency stagnation.
This predicts the accelerated convergence observed experimentally.

%% ------------------------------------------------------------------
\section{Robustness to Embodiment Mechanical Bias}
\label{app:bias_robustness}

With the DCT-II convention of Eq.~\eqref{eq:dct2_forward}, a constant offset $b$ applied to the executed sequence propagates to the DCT coefficients as
\[
  \delta\hat{c}_j\big|_\mathrm{mech}
  = 2b\sum_{n=0}^{K-1}\phi_j(n)\phi_0(n)
  = 2Kb\cdot\delta_{j0},
\]
where $\phi_0(n)\equiv 1$ and the last step follows from discrete cosine orthogonality. Only $\hat{c}_0$ (the DC coefficient, controlling mean position) is perturbed; all $j\geq 1$ shape-defining coefficients are \emph{exactly} unaffected. Under i.i.d.\ time-varying noise, by contrast, $\mathrm{Var}(\delta\hat{c}_j)=2K\sigma^2$ uniformly for all $j$, so no such isolation exists.

Under chunk FM, the same bias $b$ shifts action steps directly: $\delta a_k = b$ for $k=0,\ldots,K-1$. The perturbation spans the $K$-dimensional action space, identical in dimensionality to that induced by time-varying noise.

\section{Settings \& Additional Experiments}
\label{app:training_settingandAdditional_Experiments}

\subsection{Settings}
\label{app:training_setting}
\textbf{Standalone-policy (small-model) setting.}
Table~\ref{tab:training_setting_small} summarizes the key optimization and inference settings used in our standalone-policy experiments unless otherwise specified. Across these policy experiments, all methods use the same Transformer backbone with $6$ layers, $4$ attention heads, and embedding dimension $256$. For \motif{}, we fix the velocity regularization weight to $\lambda=1$ throughout, consistent with Eq.~\eqref{eq:total}.

\begin{table}[h]
\centering
\caption{Key optimization and inference settings for the standalone-policy
(small-model) experiments.}
\label{tab:training_setting_small}
\setlength{\tabcolsep}{8pt}
\renewcommand{\arraystretch}{1.05}
\begin{tabular*}{0.5\linewidth}{@{\extracolsep{\fill}}ll@{}}
\toprule
Setting & Value \\
\midrule
Backbone & Transformer \\
Layers & 6 \\
Attention heads & 4 \\
Embedding dimension & 256 \\
\midrule
Epochs & 3,000 (500 in toy) \\
Batch size & 256 \\
Optimizer & AdamW \\
Learning rate & $1.0\times 10^{-4}$ \\
Betas & $(0.95, 0.999)$ \\
Epsilon & $1.0\times 10^{-8}$ \\
Weight decay & $1.0\times 10^{-3}$ \\
\midrule
$\lambda$ & 1 \\
Inference sampling steps & 10 \\
Chunk horizon & 12\\
\bottomrule
\end{tabular*}
\end{table}

\textbf{Inference-step budget.}
For methods that rely on iterative sampling, we use $10$ sampling steps at inference for all compared methods. This value is slightly larger than the original setting used by MPD. We adopt the unified budget of $10$ because smaller step budgets lead to noticeably degraded performance for several methods, making cross-method comparison less stable.

\textbf{Execution protocol.}
All experiments use $10\,\mathrm{Hz}$ action execution by default. For obstacle avoidance and LapGym, all methods predict $12$ actions and execute the first $5$ before replanning every $0.5\,\mathrm{s}$. The VLA experiments follow the OpenPI receding-horizon execution protocol, predicting $50$ actions and executing the first $5$ before replanning. These matched execution settings facilitate controlled comparisons~\citep{xue2026past}.

\textbf{VLA setting.}
Table~\ref{tab:training_setting_vla} summarizes the common training settings used for the $\pi_0$ and $\pi_{0.5}$ experiments. In both cases, we use our \motif{} to train the action expert, and report the fine-tuning strategies for the VLM and action expert separately.

\begin{table}[h]
\centering
\caption{Key training settings for the VLA experiments with $\pi_0$ and
$\pi_{0.5}$.}
\label{tab:training_setting_vla}
\setlength{\tabcolsep}{8pt}
\renewcommand{\arraystretch}{1.05}
\begin{tabular*}{0.5\linewidth}{@{\extracolsep{\fill}}ll@{}}
\toprule
Setting & Value \\
\midrule
Batch size & 32 \\
Training steps & 30,000 \\
VLM & Full-parameter \\
Action expert & Full-parameter \\
Chunk horizon & 50 \\
\bottomrule
\end{tabular*}
\end{table}

\textbf{DCT coefficients. }We set the number of retained DCT coefficients to $M \approx K/3 $, where $K$ is the action chunk horizon. Specifically, $M=4$ for the standalone-policy experiments ($K=12$) and $M=16$ for the VLA experiments ($K=50$). This choice retains the dominant low-frequency components of manipulation trajectories while suppressing high-frequency noise. This choice is consistent with FAST~\citep{pertsch2025fast}, which applies DCT to robot action tokenization and reports that the analogous compression hyperparameter is insensitive. 

For mixed-frequency training, we fix $T_\xi=5\,\mathrm{s}$ and $M=16$. Although the source lengths are $K=25/50/100$ at $5/10/20\,\mathrm{Hz}$, respectively, the retained physical frequency band remains $0$-$1.6\,\mathrm{Hz}$, since $\nu_j=j/(2T_\xi)$. Thus, FAFM aligns the physical frequency band rather than the compression ratio across recording frequencies. Additional ablation results are provided in Appendix~\ref{app:additional_experiments}.

\subsection{Additional Experiments}
\label{app:additional_experiments}

\paragraph{DCT coefficient ablation and low-pass filtering.}
We evaluate the influence of the retained DCT coefficients $M$ on Rope Threading ($K=12$). We also apply low-pass filtering (LPF) to the FM baseline using the same effective bandwidth as \motif{} ($M=4$).

\begin{table}[t]
\centering
\caption{Ablation on DCT coefficients and low-pass filtering
on Rope Threading. Results are reported as mean $\pm$ std
over three training seeds.}
\label{tab:additional_M}
\begin{tabular}{lcc}
\toprule
Method & SR (\%) & LDLJ $\uparrow$ \\
\midrule
\motif{} ($M=1$) & $91.00\pm1.63$ & $-8.07\pm0.16$ \\
\motif{} ($M=2$) & $93.00\pm1.63$ & $-8.48\pm0.03$ \\
\motif{} ($M=4$) & $\mathbf{94.67\pm2.62}$
                   & $\mathbf{-7.99\pm0.30}$ \\
\motif{} ($M=6$) & $94.33\pm0.94$ & $-8.82\pm0.32$ \\
\motif{} ($M=8$) & $89.67\pm0.47$ & $-9.51\pm0.11$ \\
\midrule
FM + LPF ($M=4$) & $90.67\pm1.70$ & $-10.39\pm0.44$ \\
\bottomrule
\end{tabular}
\end{table}

As shown in Table~\ref{tab:additional_M}, $M=4$ achieves the best success rate and smoothness among the tested coefficient counts. \motif{} also outperforms the low-pass-filtered FM baseline, indicating that the improvements are not solely attributable to frequency truncation.

\paragraph{Derivative loss weight.}
We evaluate the sensitivity to the derivative regularization weight $\lambda$ on Rope Threading.

\begin{table}[t]
\centering
\caption{Sensitivity to the derivative loss weight
on Rope Threading. Results are reported as
mean $\pm$ std over three training seeds.}
\label{tab:additional_lambda}
\begin{tabular}{ccc}
\toprule
$\lambda$ & SR (\%) & LDLJ $\uparrow$ \\
\midrule
0.5 & $92.33\pm0.94$ & $-8.02\pm0.23$ \\
1.0 & $\mathbf{94.67\pm2.62}$
    & $\mathbf{-7.99\pm0.30}$ \\
2.0 & $91.67\pm2.36$ & $-8.31\pm0.03$ \\
\bottomrule
\end{tabular}
\end{table}

As shown in Table~\ref{tab:additional_lambda}, $\lambda=1$ achieves the best performance among the tested values.

\paragraph{Mixed-frequency motion smoothness.}
We further report the motion smoothness of the mixed-frequency baselines in simulation and real-world experiments, corresponding to Fig.~\ref{tab:mixed_fps}. As shown in Table~\ref{tab:mixed_ldlj}, \motif{} achieves better motion smoothness than both baselines in simulation and real-world
experiments.

\begin{table}[H]
\centering
\caption{LDLJ under mixed-frequency training.
Simulation results are reported over three training seeds.
Real-robot results are obtained from one training seed,
with standard deviations computed across evaluation episodes.}
\label{tab:mixed_ldlj}
\begin{tabular}{lcc}
\toprule
Method & Sim. LDLJ $\uparrow$ & Real LDLJ $\uparrow$ \\
\midrule
$\pi_0$-Resample & $-14.75\pm0.07$ & $-9.164\pm3.416$ \\
$\pi_0$-FreqCond & $-14.75\pm0.10$ & $-8.822\pm2.254$ \\
\motif{} & $\mathbf{-13.90\pm0.07}$
         & $\mathbf{-7.84\pm1.41}$ \\
\bottomrule
\end{tabular}
\end{table}

\paragraph{Inter-chunk consistency and RTC.}
FAFM does not explicitly enforce consistency between independently predicted action chunks. We evaluate inter-chunk smoothness using the boundary action jump:
\begin{equation}
D_{\mathrm{boundary}}
=
\frac{1}{|\mathcal{B}|}
\sum_{t\in\mathcal{B}}
\|a_t-a_{t-1}\|_2.
\end{equation}

We further integrate RTC~\citep{black2026real} as an inference-time extension of FAFM. At each flow-integration step $t$, we estimate the final DCT coefficients as
\begin{equation}
\hat c_1=c_t+(1-t)v_\theta(c_t,o,t),
\end{equation}
and decode the corresponding action trajectory
at the controller timestamps:
\begin{equation}
\hat A_1=D_{\{\tau_i\}}(\hat c_1).
\end{equation}

The RTC consistency objective is
\begin{equation}
\mathcal{L}_{\mathrm{RTC}}
=
\left\|
W\odot
\left(Y-D_{\{\tau_i\}}(\hat c_1)\right)
\right\|_2^2,
\end{equation}
where $Y$ denotes the aligned overlapping actions from the previous chunk and $W$ is the RTC soft mask. The consistency gradient is propagated through the cosine decoder to guide the coefficient-space flow state. On Rope Threading, the boundary action jump decreases from $1.21$ for FM to $0.66$ for \motif{}, and further to $0.61$ for \motif{}+RTC.

\end{document}